\documentclass[twoside]{article} 

\usepackage[accepted]{aistats2017}
\usepackage[utf8]{inputenc} 
\usepackage[T1]{fontenc}    
\usepackage{url}            
\usepackage{booktabs}       
\usepackage{amsfonts}       
\usepackage{nicefrac}       
\usepackage{microtype}      

\usepackage{xspace}
\usepackage{amssymb,amsmath,amsthm}
\usepackage{algorithm}
\usepackage{algorithmic}
\usepackage{natbib}
\usepackage{float}
\usepackage{wrapfig}

\usepackage{relsize}
\usepackage[colorinlistoftodos, textwidth=17mm, disable]{todonotes}
\usepackage{color}
\definecolor{babyblue}{rgb}{0.54, 0.81, 0.94}

\newcommand{\todomout}[1]{\todo[color=babyblue]{\tiny#1}}

\newlength{\minipagewidth}
\newlength{\minipagewidthx}
\newcommand{\bookboxxx}[1]{\small
\par\medskip\noindent
\framebox[0.45\textwidth]{
\begin{minipage}{0.42\dimexpr\textwidth-\parindent\relax} {#1} \end{minipage} } \par\medskip }

\newcommand{\forcing}{\textsc{\small{ForcingBalance}}\xspace}
\newcommand{\forcingshort}{\textsc{\small{Force}}\xspace}
\newcommand{\gafs}{\textsc{\small{GAFS-MAX}}\xspace}
\newcommand{\gafsshort}{\textsc{\small{GAFS}}\xspace}
\newcommand{\naive}{$(\mu,\sigma)$-\textsc{\small{Naive-UCB}}\xspace}
\newcommand{\ucb}{\textsc{\small{UCB}}\xspace}

\newcommand{\calP}{\mathcal P}

\newcommand{\D}{\mathcal D}

\newcommand{\E}{\mathbb E}

\renewcommand{\epsilon}{\varepsilon}

\newcommand{\var}{\sigma^2}

\newcommand{\eps}{\varepsilon}

\newcommand{\blambda}{\boldsymbol{\lambda}}

\newcommand{\CommaBin}{\mathbin{\raisebox{0.5ex}{,}}}

\newcommand{\wt}{\widetilde}
\newcommand{\wh}{\widehat}
\newcommand{\wb}{\overline}

\newtheorem{lemma}{Lemma}
\newtheorem{assumption}{Assumption}
\newtheorem{corollary}{Corollary}
\newtheorem{proposition}{Proposition}
\newtheorem{theorem}{Theorem}

\begin{document}

%

%

\twocolumn[

\aistatstitle{Trading off Rewards and Errors in Multi-Armed Bandits}

\aistatsauthor{ Akram Erraqabi  \And Alessandro Lazaric \And Michal Valko  \And Emma Brunskill \And Yun-En Liu  }

\aistatsaddress{Inria SequeL\And Inria SequeL \And Inria SequeL \And CMU \And EnLearn} ]

\begin{abstract}
In multi-armed bandits, the most common objective is the maximization of the cumulative reward. Alternative settings include active exploration, where a learner tries to gain accurate estimates of the rewards of all arms. While these objectives are contrasting, in many scenarios it is desirable to trade off rewards and errors. For instance, in educational games the designer wants to gather generalizable knowledge about the behavior of the students and teaching strategies (small \textit{estimation errors}) but, at the same time, the system needs to avoid giving a bad experience to the players, who may leave the system permanently (large \textit{reward}). In this paper, we formalize this tradeoff and introduce the \forcing algorithm whose performance is provably close to the best possible tradeoff strategy. Finally, we demonstrate on real-world educational data that \forcing returns \textit{useful} information about the arms without compromising the overall reward.
\end{abstract}



\vspace{-0.1in}
\section{Introduction}\label{s:intro}
\vspace{-0.1in}

We consider sequential, interactive systems when a learner aims at optimizing an objective function whose parameters are initially unknown and need to be estimated over time.
We take the
multi-armed bandit (MAB) framework where the learner has access to a finite
set of distributions (\textit{arms}), each one characterized
by an expected value (\textit{reward}).
The learner does not know the distributions beforehand and
it can only obtain a random sample by selecting an arm.
%
%
The most common objective in MAB is to minimize the regret, i.e.,  the difference between the reward
of the arm with the highest mean and the reward of the arms pulled by the learner.
Since the arm means are unknown,
this requires balancing \textit{exploration} of the arms and \textit{exploitation} of the mean estimates.
An alternative setting is \textit{pure exploration}, where the learner's performance is only evaluated upon the termination of
the process, and
its learning performance is allowed to be arbitrarily bad in terms of rewards accumulated over time.
In 
\emph{best-arm identification}~\citep{even-dar2006action,audibert2010best}, the learner selects arms to find the optimal arm either with very high probability or in a short number of steps. In \textit{active
exploration}~\citep{antos2010active,carpentier2011upper-confidence-bound},
the objective is to estimate the value
of all arms as accurately as possible. This setting, which is related to active learning and experimental optimal design, is particularly relevant whenever accurate predictions of the arms' value is needed to support decisions at a later time.

The previous objectives have been studied separately. However, they do not address the increasingly-prevalent situation where
users participate in research studies (e.g., for education or health) that are designed to collect reliable data and compute accurate estimates of the performance of the available options. 
Here, the subjects/users themselves rarely care about the underlying research questions but wish to gain their own benefit, such as students
seeking to learn new material, or patients seeking to
find improvement for their condition. In order to
serve these individuals and gather generalizable knowledge at the same time,
we formalize this situation as a multi-objective
bandit problem, where a designer seeks to
trade off cumulative regret minimization
(providing good direct reward for participants),
with informing scientific knowledge about the
strengths and limitations of the various conditions
(active exploration to estimate all arm means).
This tradeoff is especially needed in high-stakes domains,
such as medicine or education, or when running experiments
online, where poor experience may lead to users leaving the system permanently. 
A similar tradeoff happens in  A/B testing. Here, the designer may want to retain the ability to set a desired level of accuracy in estimating the value of different alternatives (e.g., to justify decisions that are to be taken \emph{posterior} to the experiment)  while still maximizing the reward.


A natural initial question is whether these two 
different objectives, reward maximization and 
accurate arm estimation, or other alternative 
objectives, like best arm identification, are mutually compatible:
Can one always recover the best of all objectives? 
Unfortunately, in general, the answer is negative. 
 \citet{bubeck2009pure} have already shown that any algorithm with sub-linear 
regret cannot be optimal for identifying the best arm.
Though it may not be possible to be simultaneously optimal
for both active exploration and reward maximization,
we wish to carefully trade off between these two objectives.
How to properly balance multiple objectives in MAB
is a mostly unexplored question. \citet{bui2011committing} introduce the \textit{committing bandits}, where a given horizon is divided into an experimentation phase when the learner is free to explore all the arms but still pays a cost, and a commit phase when the learner must choose one single arm that will be pulled until the end of the horizon. 
\citet{lattimore2015the-pareto} analyzes the problem where the learner wants to minimize the regret simultaneously w.r.t.\ two special arms. He shows that if the regret w.r.t.\ one arm is bounded by a small quantity $B$, then the regret w.r.t.\ the other arm scales at least as $1/B$, which reveals the difficulty of balancing two objectives at the same time. \citet{drugan2013designing} formalize the multi-objective bandit problem where each arm is characterized by multiple values and the learner should maximize a multi-objective function constructed over the values of each arm. They derive variations of  \ucb  to minimize the regret w.r.t.\ the full Pareto frontier obtained for different multi-objective functions. 
Finally, \citet{sani2012risk} study strategies having a small regret versus the arm with the best mean-variance tradeoff. In this case, they show that it is not always possible to achieve a small regret w.r.t.\ the arm with the best mean-variance.


In this paper, we study the tradeoff between cumulative reward and accuracy of estimation of the arms' values (i.e., reward maximization and active exploration), which was first introduced by~\citet{liu2014trading}. 
Their work presented a heuristic algorithm for balancing this tradeoff
and promising empirical results on an education simulation. In the present paper, we take a more rigorous approach and make several new contributions. 
\textbf{1)} We propose and justify a new objective function for the integration of rewards and estimation errors (Sect.~\ref{s:preliminaries}),
that provides a simple way for a designer to weigh directly between them.
\textbf{2)} We introduce the \forcing algorithm that optimizes the objective function when the 
arm distributions are unknown (Sect.~\ref{s:algorithms}).  
Despite its simplicity, we prove that \forcing
incurs a regret that asymptotically 
matches the minimax rate for cumulative regret minimization and the performance of active exploration algorithms (Sect.~\ref{s:theory}). 
This is very encouraging, as it shows that balancing a 
tradeoff between rewards and errors
is not fundamentally more difficult than either of these separate objectives. Interestingly, we also show 
that a simple extension of  
\ucb is not sufficient 
to achieve good performance.
\textbf{3)} Our analysis requires only requires strong convexity and smoothness of the objective function and therefore our algorithm and the proof technique can be easily extended.   
\textbf{4)} We provide 
empirical simulations on both synthetic and educational data from~\citet{liu2014trading} that support our analysis (Sect.~\ref{s:experiments}).
\section{Balancing Rewards and Errors}\label{s:preliminaries}


We consider a MAB of $K$ arms with distributions $\{\nu_i\}_{i=1}^K$, each characterized by mean $\mu_i$ and variance~$\sigma_i^2$. For technical convenience, we consider distributions with bounded support in $[0,1]$. All the following results extend to the general case of sub-Gaussian distributions (used in the experiments). 
We denote  the $s$-th i.i.d.\ sample drawn from $\nu_i$ by $X_{i,s}$ and we define $[K]=\{1,\ldots,K\}$. As discussed in the introduction, we study the combination of two objectives: reward maximization and estimation error minimization. Given a fixed sequence of $n$ arms $\mathcal{I}_n = (I_1, I_2, .., I_n)$, where $I_t\in[K]$ is the arm pulled at time $t$, the average reward is defined as
\begin{align}\label{eq:avg.reward}
\rho(\mathcal{I}_n) = \E\bigg[\frac{1}{n} \sum_{t=1}^n X_{I_t, T_{I_t,t}} \bigg] = \frac{1}{n}\sum_{i=1}^K T_{i,n} \mu_i,
\end{align}
where $T_{i,t} = \sum_{s=1}^{t-1} \mathbb{I}\{I_s = i\}$ is the number of times arm~$i$ is selected up to step $t-1$. The sequence maximizing $\rho$ simply selects the arm with largest mean for all $n$ steps.
On the other hand, the estimation error is measured as 
\begin{align}\label{eq:mse}
\varepsilon(\mathcal{I}_n) \!=\! \frac{1}{K} \!\sum_{i=1}^K \sqrt{n\E\Big[\big(\wh{\mu}_{i,n} \!-\! \mu_i\big)^2\Big]} \!=\! \frac{1}{K} \!\sum_{i=1}^K \sqrt{\frac{n\sigma_i^2}{T_{i,n}}}\CommaBin
\end{align}
%
where $\wh{\mu}_{i,n}$ is the empirical average of the $T_{i,n}$ samples. 
Similar functions were used by~\citet{carpentier2011upper-confidence-bound,carpentier2015adaptive}.
Notice that \eqref{eq:mse} is multiplying the root mean-square error by $\sqrt{n}$.
This is to allow the user
to specify a direct tradeoff between~\eqref{eq:avg.reward} and~\eqref{eq:mse}
regardless on how their average magnitude
varies as a function of~$n$.%
\footnote{This choice also ``equalizes'' the standard regret bounds for the two separate objectives, so that the minimax regret in terms of $\rho$ and the known upper-bounds on the regret w.r.t.\ $\varepsilon$ are both $\wt{O}(1/\sqrt{n})$.} 
Optimizing $\varepsilon$ requires selecting all the arms with a frequency proportional to their standard deviations. More precisely, each arm should be pulled proportionally to $\sigma_i^{2/3}$.
We define the tradeoff objective function balancing the two functions above as a convex combination,

\begin{align}\label{eq:function}
f_w(\mathcal{I}_n; &\{\nu_i\}_i) = w \rho(\mathcal{I}_n) - (1-w)\varepsilon(\mathcal{I}_n)\nonumber\\
&= w\sum_{i=1}^K \frac{T_{i,n}}{n}\mu_i - \frac{(1-w)}{K} \sum_{i=1}^K \frac{\sigma_i}{\sqrt{T_{i,n}/n}}\CommaBin
\end{align}
%

\begin{figure}[t]
\begin{center}
\includegraphics[trim={0 0.01cm 0 0.1cm},clip, width=1\columnwidth]{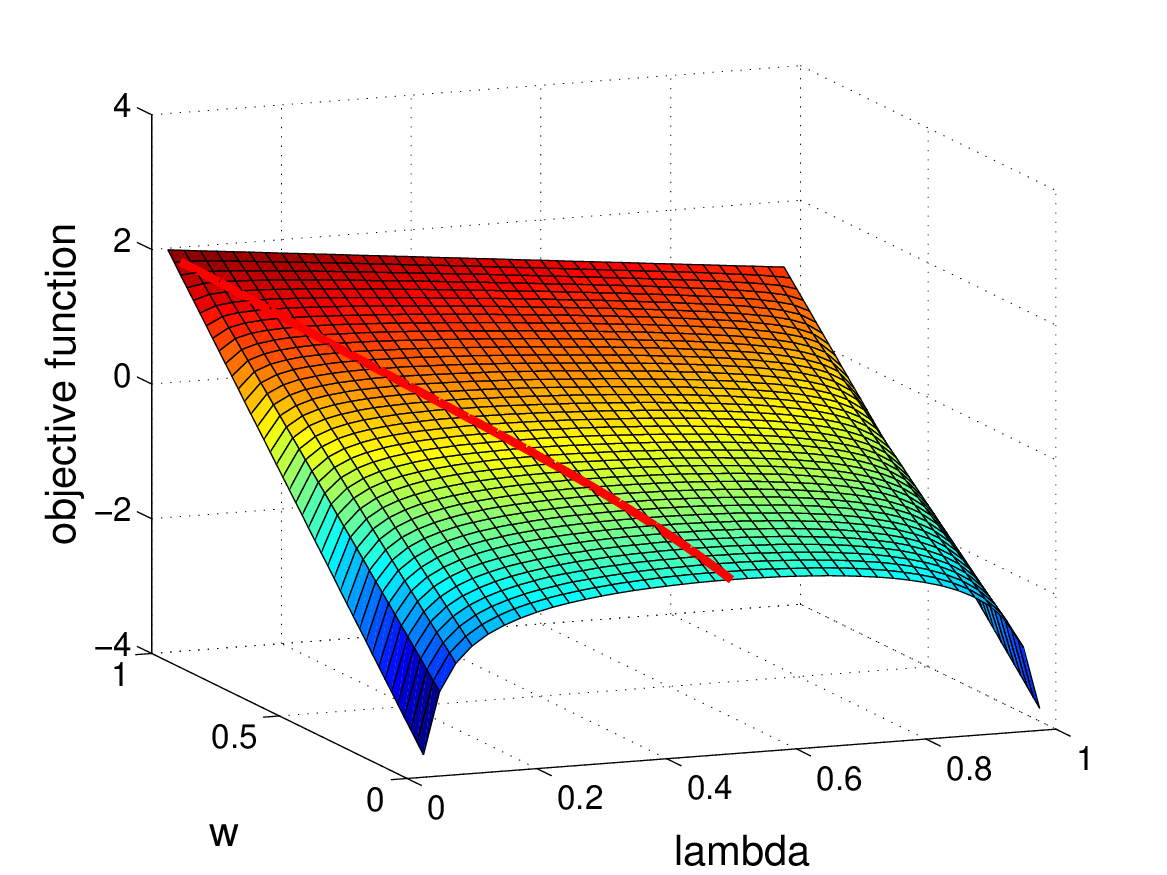}
\caption{\small Function $f_w$ and optimal solution $\blambda^*$ for different values of $w$ (\textit{red} line) for a MAB with $K=2$, $\mu_1=1$, $\mu_2=2$, $\sigma_1^2\!=\!\sigma_2^2\!=\!1$. 
For small $w$, the problem reduces to optimizing the average estimation error. Since the arms have the same variance, $\blambda^*$ is an even allocation over the two arms. As $w$ increases, the $\rho$ component in $f_w$ becomes more relevant and the optimal allocation selects arm $2$ more often, until $w=1$ when all the resources are allocated to arm $2$.}
\label{fig:objective_fun_shape}
\end{center}
\end{figure}

where $w\in[0,1]$ is a weight parameter and the objective is to find the sequence of pulls $\mathcal{I}_n$ which maximizes~$f_w$. For $w=1$, we recover the reward maximization problem, while for $w=0$, the problem reduces to minimizing the average estimation error. In the rest of the paper, we are interested in the case $w\in(0,1)$ since the extreme cases have already been studied. Using \emph{root} mean square error for $\varepsilon(\mathcal{I}_n)$ gives $f_w$ the \emph{scale-invariant} property: Rescaling the distributions equally impacts $\rho$ and $\varepsilon$. 
Furthermore, $f_w$ can be equivalently obtained as a Lagrangian relaxation of a constrained optimization problem where we intend to maximize the reward subject to a desired level of estimation accuracy. In this case, the parameter $w$ is directly related to the value of the Lagrange multiplier. \citet{liu2014trading} proposed a similar tradeoff function 
where the estimation error is measured by Hoeffding confidence intervals, which disregard the variance of the arms and only depend on the number of pulls. 
In addition, in their objective, the optimal allocation radically changes with the horizon $n$, where a short horizon forces the learner to be more explorative, while longer horizons allow the learner to be more greedy in accumulating rewards. 
Overall, their tradeoff reduces to a mixture between a completely uniform allocation (that minimizes the confidence intervals) and a \ucb strategy that maximizes the cumulative reward. While their algorithm 
demonstrated encouraging empirical performance,
no formal analysis was provided.  
In contrast, 
$f_w$ is \textit{stable} over time and it allows us to compare the performance of a learning algorithm to a static optimal allocation. We later show that $f_w$ also enjoys properties such as smoothness and strong concavity that are particularly convenient for the analysis. Besides the mathematical advantages, we notice that without normalizing $\varepsilon$~by~$n$, as $w$ tends to 0, we would never be able to recover the optimal strategy for error minimization, since $\rho(\mathcal{I}_n)$ would always dominate $f_w$, thus making the impact of tuning $w$ difficult to interpret.

Given a horizon $n$, finding the optimal $\mathcal{I}_n$ requires solving a difficult discrete optimization problem, thus we study its continuous relaxation,\footnote{A more accurate definition of $f_w$ over the simplex requires completing it with $f_w(\blambda) = -\infty$ whenever there exists a component $\lambda_i=0$ linked to a non-zero variance $\sigma_i^2$.}

\vspace{-0.2in}
\begin{align}\label{eq:function.relax}
f_w(\blambda; &\{\nu_i\}_i) = w\sum_{i=1}^K \lambda_i\mu_i - \frac{(1-w)}{K} \sum_{i=1}^K \frac{\sigma_i}{\sqrt{\lambda_i}}\CommaBin
\end{align}
\vspace{-0.2in}

where $\blambda \in \D_{K}$ belongs to the $K$-dimensional simplex such that $\lambda_i \geq 0$ and $\sum_i \lambda_i = 1$. As a result, $\blambda$ defines an allocation of arms and $f_w(\blambda; \{\nu_i\}_i)$ is its asymptotic performance if arms are repeatedly chosen according to $\blambda$. We define the optimal allocation and its performance as $\blambda^* \!=\! \arg\max_{\blambda\in\D_{K}}f_w(\blambda; \{\nu_i\}_i)$ and $f^*\! =\! f_w(\blambda^*; \{\nu_i\}_i)$ respectively.
Since $f_w$ is concave and $\D_K$ is convex, $\blambda^*$ always exists and it is unique whenever $w<1$ (and there is at least a non-zero variance)
or when the largest mean is distinct from the second largest mean. 
Although a closed-form solution cannot be computed in general, intuitively $\blambda^*$ favors arms with large means and large variance since allocating a large portion of the resources to them contribute to minimizing $f_w$ by increasing the reward $\rho$ and reducing the error $\varepsilon$. The parameter~$w$ defines the sensitivity of $\blambda^*$ to the arm parameters, such that for large $w$, $\blambda^*$ tends to concentrate on the arm with largest mean, while for small $w$, $\blambda^*$ allocates arms proportionally to their standard deviations. Fig.~\ref{fig:objective_fun_shape}, Sect.~\ref{ss:synthetic}, and App.~\ref{app:experiments} provide additional examples illustrating the sensitivity of $\blambda^*$ to the parameters in $f_w$.
Let $\mathcal{I}_n^*$ be the optimal discrete solution to Eq.~\ref{eq:function}. Then, we show that the difference between the two solutions rapidly shrinks to $0$ with $n$. In fact,\footnote{Consider a real number $r\in[0,1]$ and $R_n$ any rounding of $rn$ (e.g., $R_n = \lfloor rn\rfloor$), then $|R_n - rn| \leq 1$. If we use $\wh{r}_n = R_n/n$ as fractional approximation of $r$ with resolution $n$, then we obtain that $|\wh{r}_n - r| \leq 1/n$.} for any arm $i$, $|T^*_{i,n}/n - \lambda^*_i|\leq 1/n$ and according to Lem.~\ref{lem:error.allocation.opt} (stated later), this guarantees that the value of $\lambda^*$ ($f^*$) differs from the optimum of $f_w(\mathcal{I}_n)$ by $1/n^2$.
%

In the following, we consider the restricted simplex $\wb{\D}_K = \{\lambda_i \geq \lambda_{\min}, \sum_i \lambda_i = 1\}$ with $\lambda_{\min}\!>\!0$ on which $f_w$ is always bounded and it can be characterized by the following lemma.

%

\begin{lemma}\label{lem:strong.concave}
Let $\sigma_{\max} = \max_i \sigma_i$ and $\sigma_{\min} = \min_i \sigma_i >~0$ be the largest and smallest standard deviations, the function $f_w(\blambda; \{\nu_i\})$ is $\alpha$-strongly concave everywhere in $\D_K$ with $\alpha = \frac{3(1-w)\sigma_{\min}}{4K}$ and it is $\beta$-smooth in $\wb{\D}_K$ with $\beta = \frac{3(1-w)\sigma_{\max}}{4K\lambda_{\min}^{5/2}}\cdot$
%
%
\end{lemma}

Finally, we define the performance of a learning algorithm. 
Let $\wt{\blambda}_n$ be the empirical frequency of pulls, i.e., $\wt{\lambda}_{i,n} = T_{i,n}/n$. We define its regret w.r.t.\ the value of the optimal allocation as $R_n(\wt{\blambda}_n) = f^* - f_w(\wt{\blambda}_n; \{\nu_i\}_i)$.
The previous equation defines the pseudo-regret of a strategy $\wt{\blambda}_n$, since in Eq.~\ref{eq:mse} the second equality is true for fixed allocations. This is similar to the definition of~\citet{carpentier2015adaptive}, where the difference between \textit{true} and pseudo-regret is discussed  in detail.


\vspace{-0.1in}
\section{The \forcing Algorithm}\label{s:algorithms}
\vspace{-0.1in}

\textbf{Why na\"{i}ve \ucb fails.}
One of the most successful approaches to bandits is the optimism-in-face-of-uncertainty, where we construct confidence bounds for the parameters and select the arm maximizing an upper-bound on the objective function. This approach was successfully applied in both regret minimization (see e.g.,~\cite{auer2002finite-time}) and active exploration (see e.g.,~\cite{carpentier2011upper-confidence-bound}). As such, a first natural approach to our problem is to construct an upper-bound on $f_w$ as (see Prop.~\ref{prop:confidence.intervals} for the definition of the confidence bounds)
\begin{align}\label{eq:naive.ucb}
f_w^{UB}(\blambda; \{\wh{\nu}_{i,n}\}&) = w \sum_{i=1}^K \lambda_i \bigg(\wh{\mu}_{i,n} + \sqrt{\frac{\log(1/\delta_n)}{2T_{i,n}}}\bigg) \\
&\!\!\!-\! (1\!-\!w) \sum_{i=1}^K \frac{1}{\sqrt{\lambda_i}} \bigg(\wh{\sigma}_{i,n} \!-\! \sqrt{\frac{2\log(2/\delta_n)}{T_{i,n}}}\bigg)\cdot\nonumber
\end{align}
At each step $n$, we compute the allocation $\wh{\blambda}_{i,n}^{UB}$ maximizing $f_w^{UB}$ and select arms accordingly (e.g., by pulling an arm at random from $\wh{\blambda}_{i,n}^{UB}$). Although the confidence bounds guarantee that for any $\blambda$, $f_w^{UB}(\blambda; \{\wh{\nu}_{i,n}\}) \geq f_w(\blambda; \{\nu_i\})$ w.h.p.,  this approach is intrinsically flawed and it would perform poorly. While for large values of $w$, the algorithm reduces to \ucb, for small values of $w$, the algorithm tends to allocate arms to balance the estimation errors on the basis of \textit{lower-bounds} on the variances and thus arms with small lower-bounds are selected less. Since small lower-bounds may be associated with arms with large confidence intervals, and thus poorly estimated variances, this behavior would prevent the algorithm from correcting its estimates and improving its performance over time (see App.~\ref{app:experiments} for additional discussion and empirical simulations).
Constructing lower-bounds on~$f_w$  suffers from the same issue. This suggests that a \emph{straightforward} (na\"{i}ve) application of a \ucb-like strategy fails in this context. As a result, we take a different approach and  propose a \textit{forcing} algorithm inspired by the \gafs algorithm introduced by~\citet{antos2010active} for active exploration.\footnote{Variations on the \textit{forcing} or \textit{forced sampling} approach have been used in many settings including standard bandits~\citep{yakowitz1995the-nonparametric,szepesvari2008learning}, linear bandits~\citep{goldenshluger2013a-linear}, contextual bandits~\citep{langford2007the-epoch-greedy}, and experimental optimal design~\citep{wiens2014v-optimal}.}

\begin{figure}[t]
\begin{center}
\bookboxxx{
\begin{small}
    \begin{algorithmic}[1]
        \renewcommand{\algorithmicrequire}{\textbf{Input:}}
        \renewcommand{\algorithmicensure}{\textbf{Output:}}
\STATE \textbf{Input:} forcing parameter $\eta$, weight $w$
\FOR{$t=1,\ldots,n$}
\STATE $U_t = \arg\min T_{i,t}$
\IF{$T_{U_t,t} < \eta\sqrt{t}$}
\STATE Select arm $I_t = U_t$ (\textbf{forcing})
\ELSE
\STATE Compute optimal estimated allocation
\vspace{-0.1in}
$$ \wh{\blambda}_t = \arg\max_{\blambda\in\wb{\D}_K} f_w(\blambda; \{\wh{\nu}_{i,t}\}_i) $$
\vspace{-0.1in}
\STATE Select arm (\textbf{tracking})
\vspace{-0.1in}
$$ I_t = \arg\max_{i=1,\ldots,K} \wh\lambda_{i,t} - \wt\lambda_{i,t} $$
\vspace{-0.15in}
\ENDIF
\STATE Pull arm $I_t$, observe $X_{I_t,t}$, update $\wh\nu_{I_t}$.
\ENDFOR
    \end{algorithmic}
    \vspace{-0.05in}
    \caption{The \forcing algorithm.}
    \label{alg:forcing}
\end{small}
}
\vspace{-0.15in}
\end{center}
\end{figure}

%
%

\textbf{Forced sampling.}
The \forcing algorithm is illustrated in Fig.~\ref{alg:forcing}. It receives as input an exploration parameter $\eta>0$ and the restricted simplex $\wb{\D}_K$ defined by $\lambda_{\min}$. At each step $t$, the algorithm first checks the number of pulls of each arm and selects any arm with less than $\eta\sqrt{t}$ samples. If all arms have been sufficiently pulled, the allocation~$\wh{\blambda}_t$ is computed using the empirical estimates of the arms' means and variances $\wh{\mu}_{i,t} = \frac{1}{T_{i,t}} \sum_{s=1}^{T_{i,t}} X_{i,s}$ and $\wh{\sigma}^2_{i,n} = \frac{1}{2T_{i,t}(T_{i,t}-1)} \sum_{s,s'=1}^{T_{i,t}} \big(X_{i,s} - X_{i,s'}\big)^2$.
Notice that the optimization is done over the restricted simplex $\wb{\D}_K$ and $\widehat\blambda_t$ can be computed efficiently. 
Once the allocation $\wh{\blambda}_t$ is computed, an arm is selected. A straightforward option is either to directly implement the optimal estimated allocation by pulling an arm drawn at random from it or allocate the arms proportionally to $\wh{\blambda}_t$ over a short phase.
Both solutions may not be effective since the final performance is evaluated according to the \textit{actual} allocation realized over all 
$n$ steps (i.e., $\wt{\lambda}_{i,n} = T_{i,n}/n$) and 
not $\wh{\blambda}_n$. Consequently, even when $\wh{\blambda}_n$ is an accurate approximation of $\blambda^*$, the regret may not be small.\footnote{Consider the case of 3 arms, where 
after $t$ steps, the empirical allocation $\wt{\blambda}_t$ is $(0.5,0.1,0.4)$ and 
the estimated allocation $\wh{\blambda}_t$ is $(0.5,0.4,0.1)$. 
In the following steps, the most effective way to reduce the regret is not to use $\wh{\blambda}_t$, 
but to pull arm 2 more than $40\%$,
in order to close the gap between $\wt{\blambda}_t$ and $\wh{\blambda}_t$ as fast as possible.}  \forcing explicitly tracks the allocation $\wh{\blambda}_n$ by selecting the arm $I_t$ that is under-pulled the most so far.
This tracking step allows us to force $\wt{\blambda}_{n}$ to stay close to $\wh{\blambda}_n$ (and its performance) at each step. 
The tracking step is slightly different from \gafs, which selects the arms with largest ratio between $\wh{\lambda}_{i,n}$ and $\wt{\lambda}_{i,t}$. We show in the analysis that the proposed tracking rule is more efficient.

The parameter $\eta$ defines the amount of exploration forced by the algorithm. A large $\eta$ forces all arms to be pulled many times. While this guarantees accurate estimates $\wh{\mu}_{i,t}$ and $\wh{\sigma}^2_{i,n}$ and an optimal estimated allocation $\wh{\blambda}_t$ that rapidly converges to $\blambda^*$, the algorithm would perform the tracking step very rarely and thus~$\wt{\blambda}_t$ would not track $\wh{\blambda}_t$ fast enough. 
In the next section, we show that any value of $\eta$ in a wide range (e.g., $\eta=1$) guarantees a small regret. The other parameter is $\lambda_{\min}$ which defines a restriction on the set of allocations that can be learned. From an algorithmic point of view, $\lambda_{\min} = 0$ is a viable choice since $f_w$ is strongly concave and it always admits at least one solution in $\D_K$ (the full simplex). 
Nonetheless, we show next that $\lambda_{\min}$ needs to be strictly positive to guarantee uniform convergence of $f_w$ for true and estimated parameters, which is a critical property to ensure regret bounds.


\vspace{-0.1in}
\section{Theoretical Guarantees}\label{s:theory}
\vspace{-0.1in}

In this section, we derive an upper-bound on the regret of \forcing with explicit dependence on its parameters and the characteristics of $f_w$. 
%
We start with high-probability confidence intervals for the mean and the standard deviation (see Thm.\@ 10 of~\cite{maurer2009empirical}).

\begin{proposition}\label{prop:confidence.intervals}
Fix $\delta\in(0,1)$. For any $n>0$ and any arm $i\in[K]$, $\big| \wh{\mu}_{i,n} \!-\! \mu_i \big| \leq \sqrt{\frac{\log(1/\delta_n)}{2T_{i,n}}},\;\; \big| \wh{\sigma}_{i,n} \!-\! \sigma_i \big| \leq \sqrt{\frac{2\log(2/\delta_n)}{T_{i,n}}},$
%
%
w.p.~$1-\delta$, where $\delta_n = \delta/(4Kn(n+1))$.
\end{proposition}


The accuracy of the estimates translates into the difference in estimated and true function $f$ (we drop the dependence on $w$ for readability).

\begin{lemma}\label{lem:error.distro}
Let $\wh{\nu}_{i}$ be an empirical distribution characterized by mean $\wh{\mu}_{i}$ and variance $\wh{\sigma}_{i}$ such that $|\wh{\mu}_{i} - \mu_i| \leq \epsilon^\mu_{i}$ and $|\wh{\sigma}_{i} - \sigma_i| \leq \epsilon^\sigma_{i}$, then for any fixed $\blambda\in \D_K$ we have $\big|f(\blambda; \{\nu_i\}) \!-\! f(\blambda; \{\wh{\nu}_{i}\})\big| \leq w \max_i \epsilon^\mu_{i} + \frac{1-w}{\min_i \sqrt{\lambda_{i}}} \max_i \epsilon^\sigma_{i}$.
%
%
\end{lemma}

This lemma shows that the accuracy in estimating $f$ is affected by the largest error in estimating the mean or the variance of any arm. This is due to the fact that $\blambda$ may give a high weight to a poorly estimated arm, i.e., $\lambda_i$ may be large for large $\epsilon_i$. As a result, if $\epsilon^\mu_{i}$ and $\epsilon^\sigma_{i}$ are defined as in Prop.~\ref{prop:confidence.intervals}, the lemma requires that all arms are pulled often enough to guarantee an accurate estimation of~$f$. Furthermore, the upper-bound scales inversely with the minimum proportion $\min_i \lambda_i$. This shows the need of restricting the possible $\blambda$s to allocations with a non-zero lower-bound to $\min_i \lambda_i$, which is guaranteed by the use of the restricted simplex $\wb{\D}_K$ in the algorithm. Finally, notice that here we consider a fixed allocation $\blambda$, while later we need to deal with a (possibly) random choice of $\blambda$, which requires a union bound over a cover of $\wb{\D}_K$ (see Cor.~\ref{cor:error.distro.ext}).
%
%
Next two lemmas show how the difference in performance translates in the difference of allocations and vice versa.

\begin{lemma}\label{lem:error.inversion}
If an allocation $\blambda\in\D_K$ is such that $\big|f^* - f(\blambda; \{\nu_i\})\big| \leq \epsilon^f$, then for any arm $i\in[K]$, $|\lambda_i - \lambda^*_i| \leq \sqrt{\frac{2K}{\alpha}} \sqrt{\epsilon^f}$,
where $\alpha$ is the strong-concavity parameter of $f_w$ (Lem.~\ref{lem:strong.concave}).
\end{lemma}

\begin{lemma}\label{lem:error.allocation.opt}
The performance of an allocation $\blambda\in\wb{\D}_K$ compared to the optimal allocation $\blambda^*$ is such that $f(\blambda^*; \{\nu_i\}) - f(\blambda; \{\nu_i\}) \leq \frac{3\beta}{2}\|\blambda - \blambda^*\|^2$.
\end{lemma}

In both cases, the bounds depend on the shape of~$f$ through the parameters of strong concavity $\alpha$ and smoothness $\beta$, which in turn depends on the constrained simplex $\wb{\D}_K$ and the choice of $\lambda_{\min}$.
%
Before stating the regret bound, we need to introduce an assumption on $\blambda^*$.

\begin{assumption}\label{asm:simplex}
Let $\lambda_{\min}^* = \min_i \lambda^*_i$ be the smallest proportion over the arms in the optimal allocation and let $\wb{\D}_K$ the restricted simplex used in the algorithm. We assume that the weight parameter $w$ and the distributions $\{\nu_i\}_i$ are such that $\lambda_{\min}^* \geq \lambda_{\min}$, that is $\blambda^*\in\wb{\D}_K$.
\end{assumption}

Notice that whenever all arms have non-zero variance and $w<1$, $\lambda_{\min}^*>0$ and there always exists a non-zero $\lambda_{\min}$ (and thus a set $\wb{\D}_K$) for which the assumption can be verified. In general, the larger and more similar the variances and the smaller $w$, the bigger $\lambda_{\min}^*$ and less restrictive the assumption. The choice of $\lambda_{\min}$ also affects the final regret bound.

\begin{theorem}\label{thm:regret}
We consider a MAB with $K \geq 2$ arms with mean $\{\mu_i\}$ and variance $\{\var_i\}$. Under Asm.~\ref{asm:simplex}, \forcing with a parameter $\eta\leq 21$ and a simplex $\wb{\D}_K$ restricted to $\lambda_{\min}$ suffers a regret 
\begin{align*}
R_n(\wt{\blambda}) \leq
\begin{cases}
1 & \mbox{if } n \leq n_0 \\
\mathlarger{43 K^{5/2} \frac{\beta}{\alpha} \sqrt{\frac{\log(2/\delta_n)}{\eta\lambda_{\min}}} n^{-1/4}} &\mbox{if } n_0 < n \leq n_2\\
\mathlarger{153 K^{5/2} \frac{\beta}{\alpha} \sqrt{\frac{\log(2/\delta_n)}{\lambda_{\min}\lambda_{\min}^*}}}n^{-1/2} & \mbox{if } n > n_2,
\end{cases}
\end{align*}
\vspace{-0.1in}

with probability $1-\delta$ (where $\delta_n =  \delta/(4Kn(n+1))$) and
\begin{align*}
n_0 &= K(K\eta^2 + \eta\sqrt{K} + 1),\\
n_2 &= \frac{C}{(\lambda_{\min}^*)^8} \frac{K^{10}}{\alpha^4} \frac{\log^2(1/\delta_n)}{\lambda_{\min}^2}\CommaBin
\end{align*}
where $C$ is a numerical constant.
\end{theorem}

\textbf{Remark 1 (dependence on $n$).} 
The previous bound reveals the existence 
of three phases. For $n\leq n_0$, we are in a fully explorative phase where the pulls are always triggered by the forcing condition, 
the allocation $\wt{\blambda}_n$ is uniform over arms; and 
it can be arbitrarily bad w.r.t.\ $\blambda^*$. 
In the second phase, the algorithm interleaves forcing and tracking but the estimates $\{\wh{\nu}_{i,n}\}$ are not accurate enough to guarantee that $\wh{\blambda}_n$ performs well. In particular, 
we can only guarantee that all arms are selected $\eta\sqrt{n}$, which implies the regret decreases very slowly as $\wt{O}(n^{-1/4})$. 
Fortunately, as the estimates become more accurate, $\wh{\blambda}_n$ approaches $\blambda^*$, and
 after $n_2$ steps the algorithm successfully tracks 
$\blambda^*$ and achieves the asymptotic regret of $\wt{O}(n^{-1/2})$. This regret 
matches the minimax rate for regret minimization and active exploration (e.g., \gafs). 
This shows that operating a trade-off between rewards and errors is not fundamentally more difficult than optimizing either of the objectives individually. While in this analysis, the second and third phases are sharply separated (and $n_2$ may be large), in practice the performance gradually improves as $\wh\blambda$ approaches $\blambda^*$.
%


\textbf{Remark 2 (dependence on parameters).} 
$\lambda_{\min}$ has a major impact on the bound.
The smaller its value, the higher the regret, both explicitly and through the smoothness $\beta$. At the same time, the larger $\lambda_{\min}$ the stricter Asm.~\ref{asm:simplex}, which limits the validity of Thm.~\ref{thm:regret}. A possible compromise is to set $\lambda_{\min}$ to an appropriate decreasing function of $n$, thus making Asm.~\ref{asm:simplex} always verified (for a large enough $n$), at the cost of worsening the rate of the regret. In the experiments, we run \forcing with $\lambda_{\min}\!=\!0$ without the regret being negatively affected. We conjecture that we can always set $\lambda_{\min}\!=\!0$ (for which Asm.~\ref{asm:simplex} is always verified), while the bound could be refined by replacing $\lambda_{\min}$ (the \forcing parameter) with $\lambda_{\min}^*$ (the minimum optimal allocation). 
Nonetheless, we point out that this would require to significantly change the structure of the proof as Lem.~\ref{lem:error.distro} does not hold anymore when $\lambda_{\min}=0$.

\textbf{Remark 3 (dependence on the problem).} The remaining terms in the bound depend on the number of arms $K$, $w$, $\sigma_{\min}^2$ (through $\alpha$), and $\lambda_{\min}^*$. By definition of $\alpha$, we notice that as $w$ tends to 1 (pure reward maximization), the bound gets worse. This is expected since the proof relies on the strong convexity of $f_w$ to relate the accuracy in estimating $f_w$ and the accuracy of the allocations (see Lem.~\ref{lem:error.inversion}). Finally, the regret has an inverse dependence on $\lambda_{\min}^*$, which shows that if the optimal allocation requires an arm to be selected only a very limited fraction of the time, the problem is more challenging and the regret increases. This may happen in a range of different configurations such as the large value of $w$ or when one arm has very high mean and variance, which leads to a $\blambda^*$ highly concentrated on one single arm and $\lambda_{\min}^*$ very small. A very similar dependence is already present in previous results for active exploration (see e.g.,~\cite{carpentier2011upper-confidence-bound}). 

\textbf{Remark 4 (proof).} A sketch and the complete proof are in App.~\ref{s:thm.proof}. While the proof shares a similar structure as \gafs's, in \gafs we have access to an explicit form of the optimal allocation $\blambda^*$ and the proof directly measures the difference between allocations. Here, we have to rely on Lemmas~\ref{lem:error.inversion} and~\ref{lem:error.allocation.opt} to relate allocations to objective functions and vice versa. In this sense, our analysis is a generalization of the proof in \gafs and it can be applied to any strongly convex and smooth objective function.


\vspace{-0.1in}
\section{Experiments}\label{s:experiments}
\vspace{-0.1in}

We evaluate \forcing on synthetic data and a problem directly derived 
from an educational application. 
Additional experiments are in the appendix.

\subsection{Synthetic Data}\label{ss:synthetic}



\begin{figure*}[t]
\vspace{-0.1in}
\centering
\includegraphics[trim={2.5cm 6.0cm 2cm 7cm},clip,width=0.32\textwidth]{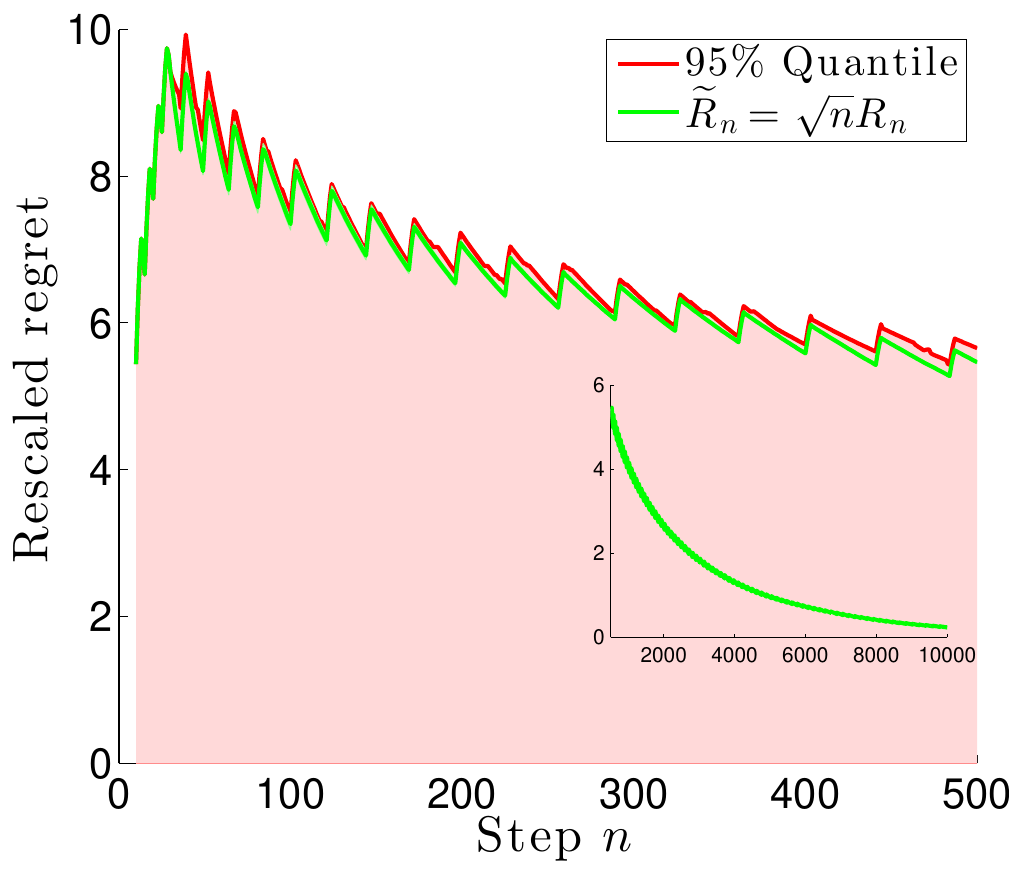}
\includegraphics[trim={1.0cm 0cm 0.1cm 0cm},clip, width=0.33\textwidth]{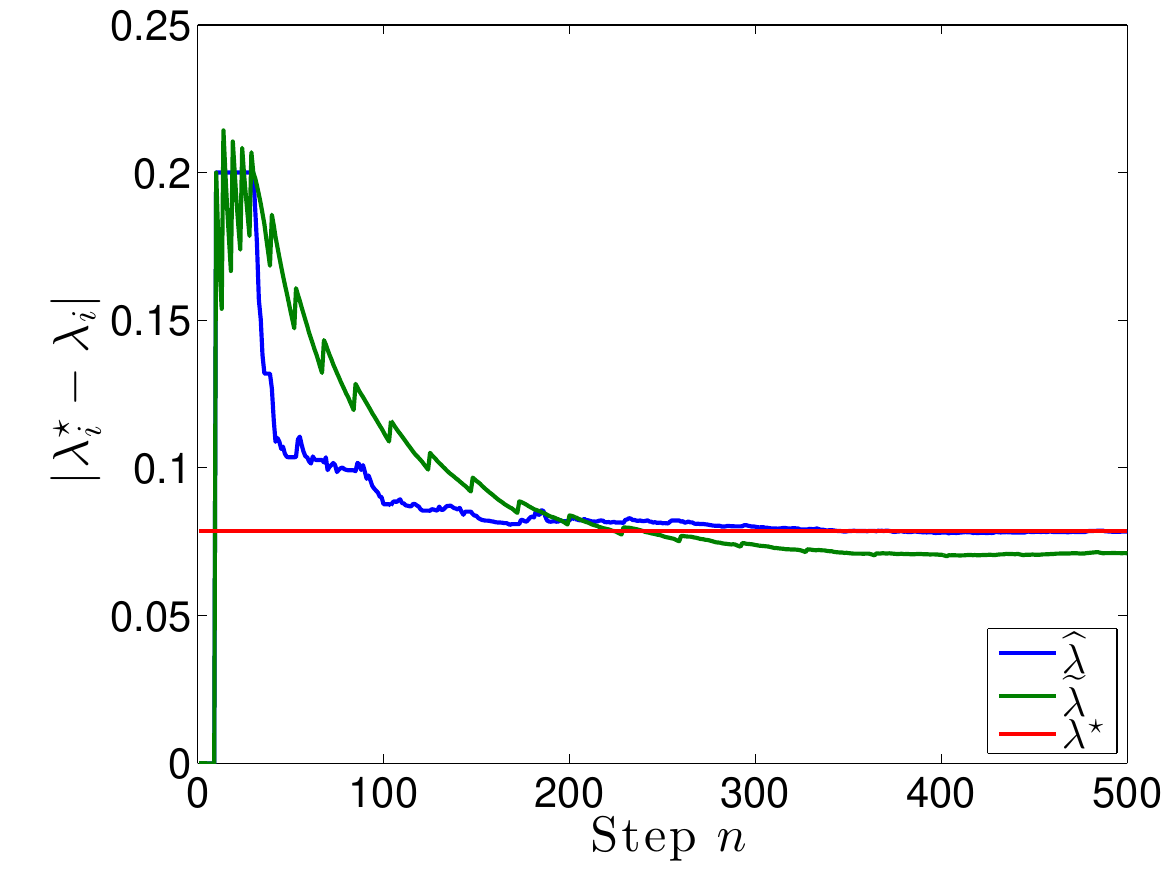}
\includegraphics[trim={0.25cm 0cm 0cm 0cm},clip,width=0.33\textwidth]{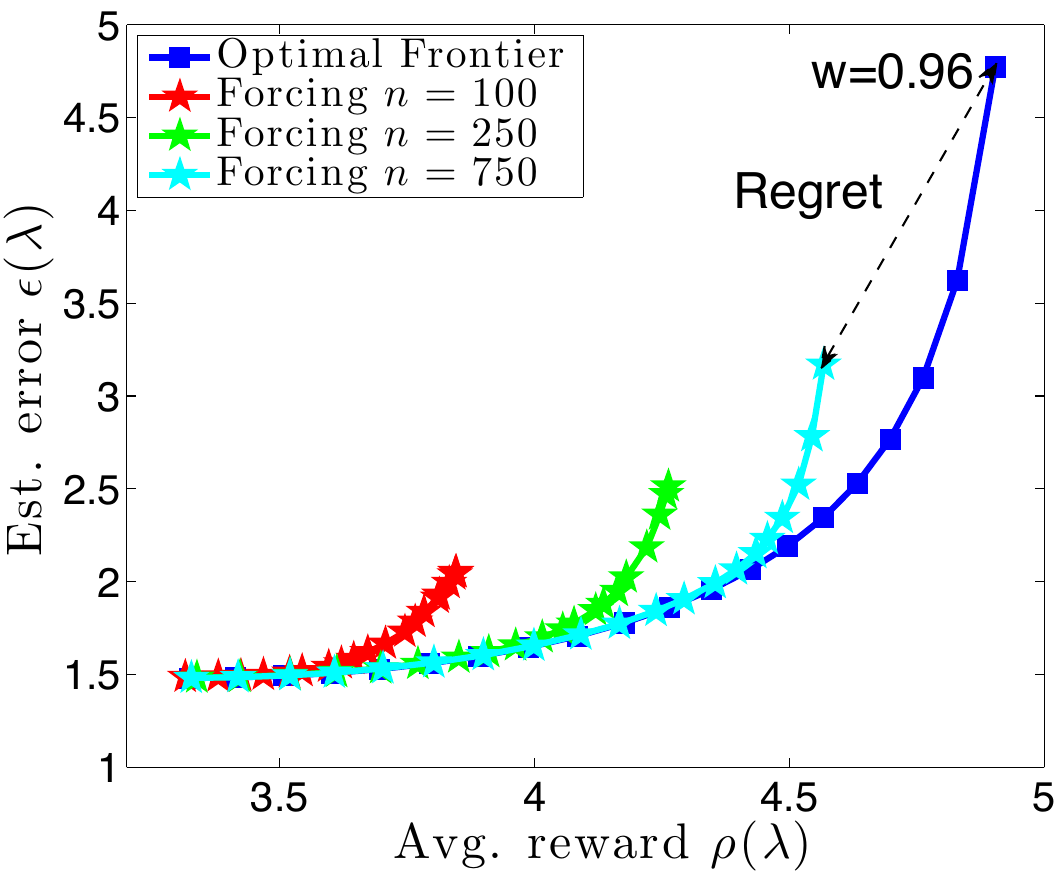}
\vspace{-0.15in}
\caption{\small Rescaled regret \textit{(left)}, allocations errors \textit{(center)}, Pareto frontier \textit{(right)} for the setting in Fig.~\ref{tab:exp.setting}.}
\vspace{-0.1in}
\label{fig:exp3}
\end{figure*}


\begin{figure}[h!]
\vspace{15pt}
\begin{small}
\centering
\renewcommand{\arraystretch}{1.1}
\begin{tabular}{| c || c | c | c |}
\hline
 & $\mu$ & $\var$ & $\blambda^*$ \\
\hline
\hline
Arm1 & 1.0 & 0.05 & 0.0073 \\
\hline
Arm2 & 1.5 & 0.1 & 0.01 \\
\hline
Arm3 & 2.0 & 0.2 & 0.014 \\
\hline
Arm4 & 4.0 & \textbf{4.0} & 0.0794 \\
\hline
Arm5 & \textbf{5.0}  & 0.5 &   0.8893 \\
\hline
\end{tabular}
\caption{\small Arm mean, variance and optimal allocation for $w=0.9$.}
\label{tab:exp.setting}
\end{small}
\end{figure}

We consider a MAB with $K=5$ arms with mean and variance given in Fig.~\ref{tab:exp.setting}. While $\rho(\blambda)$ is optimized by always pulling  arm $5$, $\varepsilon(\blambda)$ is minimized by an allocation  selecting more often arm $4$ that has the larger variance (for $w=0$, the optimal allocation $\lambda_4^*$ is over $0.41$). For $w=0.9$ (i.e., more weight to cumulative reward than estimation error) the optimal allocation $\blambda^*$ is clearly biased towards arm $5$ and only partially to arm $4$, while all other arms are pulled only a limited fraction of time (well below $2\%$). We run \forcing with $\eta=1$ and $\lambda_{\min}=0$ and we average over 200 runs.

\textbf{Dependence on $n$.}
In Fig.~\ref{fig:exp3}-\textit{(left)} we report the average and the $0.95$-quantile of the rescaled regret $\wt{R}_n = \sqrt{n} R_n$. From Thm.~\ref{thm:regret} we expect the rescaled regret to increase as $\sqrt{n}$ in the first exploration phase, then to increase as $n^{1/4}$ in the second phase, and finally converge to a constant (i.e., when the actual regret enters into the asymptotic regime of $\wt{O}(n^{-1/2})$). From the plot we see that this is mostly verified by the empirical regret, although there is a transient phase during which the rescaled regret decreases over $n$, which suggests that the actual regret may decrease with a faster rate, at least in a first moment. This behavior may be captured in the theoretical analysis by replacing the use of Hoeffding bounds with Bernstein concentration inequalities, which may reveal faster rate (up to $\wt{O}(1/n)$) whenever $n$ and the standard deviations are small. 

\textbf{Tracking.}
In Fig.~\ref{fig:exp3}-\textit{(center)}, we study the behavior of the estimated allocation $\wh{\blambda}$ and the actual allocation $\wt{\blambda}$ (we show $\wh{\lambda}_4$ and $\wt{\lambda}_4$) w.r.t.~the optimal allocation ($\lambda_4^* = 0.0794$). In the initial phase, $\wh{\blambda}$ 
is basically uniform ($1/K$) since the algorithm is always in forcing mode. 
After the exploration phase, $\wh{\blambda}$ is computed on the estimates that are already quite accurate, and it rapidly converges to $\blambda^*$. At the same time, $\wt{\blambda}$ keeps tracking the estimated optimal allocation and it also tends to converge to $\blambda^*$ but with a slightly longer delay.
We further study the tracking rule in the appendix.

%
%
\textbf{Pareto frontier.}
In Fig.~\ref{fig:exp3}-\textit{(right)} we study the performance of the optimal allocation $\blambda^*$ for varying weights~$w$. We report the Pareto frontier in terms of average reward $\rho(\blambda)$ and average estimation error $\varepsilon(\blambda)$. The optimal allocation smoothly changes from focusing on arm 4 to being almost completely concentrated on arm 5 ($\lambda^*_4 = 0.41$ and $\lambda^*_5 = 0.20$ for $w=0.0$ and $\lambda^*_4=0.0484$ and $\lambda^*_4=0.9326$ for $w=0.95$). As a result, we move from an allocation with very low estimation error but poor reward to a strategy with large reward but poor estimation. We report the Pareto frontier of \forcing for different values of $n$. In this setting, \forcing is more effective in approaching the performance of $\blambda^*$ for small values of $w$. This is consistent with the fact that for $w=0$, $\lambda_{\min}^* = 0.097$, while it decreases to $0.004$ for $w=0.95$, which increases the regret as illustrated by Thm.~\ref{thm:regret}.



\vspace{-0.1in}
\subsection{Educational Data}\label{ss:education}
\vspace{-0.1in}

%


\begin{figure*}[t!]

\vspace{-3pt}
\begin{minipage}{0.45\textwidth}
\begin{center}
\includegraphics[width=0.89\columnwidth]{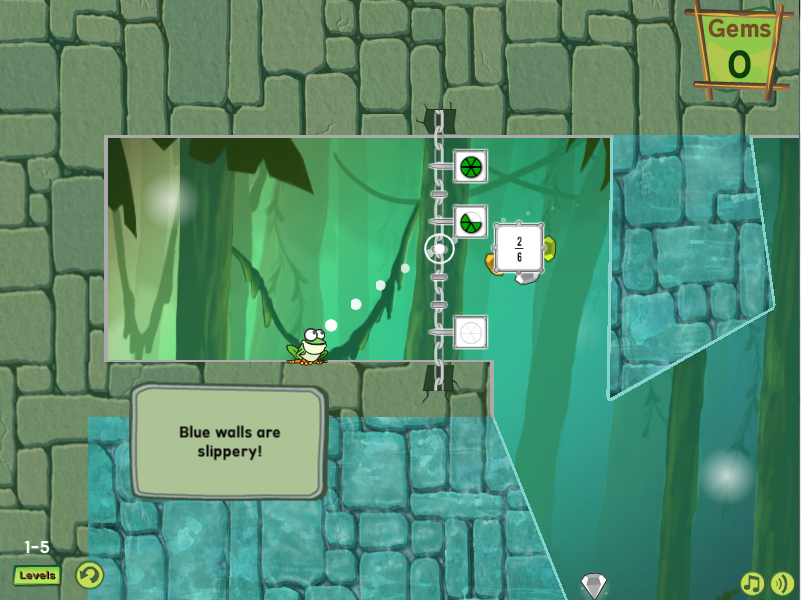}
\caption{Treefrog Treasure, a math game about number lines.}
\label{fig:Treefrog}
\vspace{-4pt}
\end{center}
\end{minipage}
\begin{small}
\begin{minipage}{0.6\textwidth}
\centering
\renewcommand{\arraystretch}{1.1}
\begin{tabular}{|c|c|c|c|c|c|} 
\hline
 Alg. & {\small$\frac{\varepsilon(\blambda)}{\var_{\max}}$} & {\small$\frac{\rho(\blambda)}{\mu_{\max}}$} & {\small$R_n$} & {\small\textit{RelDCG}} & {\small\textit{RankErr}} \\ 
 \hline
 \hline
\multicolumn{6}{|c|}{$w=0.95$} \\
\hline
$\blambda^*$ & 6.549 & 0.9405 & - & - & - \\ 
 \hline 
\forcingshort & 6.708 & 0.9424 & \textbf{1.878} & 0.1871 & 5.935 \\ 
 \hline 
 \ucb & 11.03 & \textbf{0.9712} & 95.15 & 1.119 & 8.629 \\  
 \hline 
\gafsshort & \textbf{5.859} & 0.9183 & 17.79 & \textbf{0.1268} & \textbf{5.117} \\ 
 \hline 
\textit{Unif} & 5.861 & 0.9168 & 20.49 & 0.132 & 5.25 \\ 
 \hline 
\hline
\multicolumn{6}{|c|}{$w=0.6$} \\
\hline
$\blambda^*$ & 5.857 & 0.9189 & - & - & - \\ 
 \hline 
\forcingshort & \textbf{5.859} & 0.92 & \textbf{0.4437} & \textbf{0.1227} & 5.178 \\ 
 \hline 
\ucb & 11.03 & \textbf{0.9712} & 1343 & 1.119 & 8.629 \\ 
 \hline 
\gafsshort & \textbf{5.859} & 0.9183 & 1.314 & 0.1268 & \textbf{5.117} \\ 
 \hline 
\textit{Unif} & 5.861 & 0.9168 & 3.482 & 0.132 & 5.25 \\ 
 \hline 
 \end{tabular}

\caption{Results on the educational dataset.}
\label{tab:educ.results}
\end{minipage}
\end{small}
\vspace{-0.10in}
\end{figure*}

\textit{Treefrog Treasure} is an educational math game in which players navigate through a world of number lines. Players must find and jump through the appropriate fraction on each number line. 
To analyze the effectiveness of our algorithm when parameters are drawn from a real-world setting, we use data from an experiment in \textit{Treefrog Treasure} to estimate the means and variances of a 64-arm experiment. Each arm corresponds to a different experimental condition: After a tutorial, 34,197 players each received a pair of number lines with different properties, followed by the same (randomized) distribution of number lines thereafter. We measured how many number lines students solved conditioned on the type of this initial pair; the hope is to learn which type of number line encourages player persistence on a wide variety of number lines afterwards.
There were a total of $K=64$ conditions, formed from choosing between 2 representations of the target fraction, 2 representations of the label fractions on the lines themselves, adding or withholding tick marks at regular intervals on the number line, adding or removing hinting animations if the problem was answered incorrectly, and 1-4 different rates of backoff hints that would progressively offer more and more detailed hints as the player made mistakes. The details of both the experiments and the experimental conditions are taken from~\citet{liu2014trading}%
, though we emphasize that we measure a different outcome in this paper (player persistence as opposed to chance of correct answer).

We run \forcing, standard \ucb, \gafs (adapted to minimize the average estimation error) over $n=25,000$ and 100 runs. Both \forcing and \gafs use $\eta=1$ and $w$ is set to 0.6 to give priority to preference to the accuracy of the estimates and to 0.95 to favor the player's experience and entertainment.
We study the performance according to the average reward $\rho(\blambda)$ (normalized by the largest mean), the estimation error $\varepsilon(\blambda)$ (normalized by the largest standard deviation), the rescaled regret $\sqrt{n}R_n$, the relative \textit{discounted cumulative gain} (DCG) and the \textit{RankErr} that measure how well arms are ranked on the basis of their mean.\footnote{Let $\pi^*$ be the true ranking and $\wh{\pi}$ the estimated ranking (i.e., $\wh{\pi}(k)$ returns the identity of the arm ranked at position $k$), the DCG is computed as $\text{DCG}_{\pi} = \sum_{k=1}^K \frac{\mu_{\pi(k)}}{\log(k+1)}$ and then we compute $(\text{DCG}_{\pi^*} - \text{DCG}_{\wh{\pi}})/\text{DCG}_{\pi^*}$, while $RankErr = 1/K\sum_{i=1}^K |\pi^*(i) - \wh{\pi}(i)|$.} Small values of \textit{RelDCG} and \textit{RankErr} mean that arms are estimated well enough to correctly rank them and can allow 
the experiment designer to later 
reliably remove the worst performing arms. 
The results are reported in Fig.~\ref{tab:educ.results}. Since \ucb, \gafs, and \textit{Unif} do not depend on $w$, their performance is constant except for the regret, which is computed w.r.t.\ to different $\blambda^*$s. As expected, \ucb achieves the highest reward but it performs very poorly in estimating the arms' mean and in ranking them. \gafs does not collect much reward but is very accurate in the estimate of the means. On the other hand, \forcing balances the two objectives and it achieves the smallest regret. We notice that \forcing preserves a very good estimation accuracy without compromising too much the average reward (for $w=0.95$).
%
In this situation, effectively balancing between the two objectives allows us to rank the different game settings in the right order while providing players with a good experience. Had we used \ucb, the outcome for players would have been better, but the designers would have less insight into how the different ways of providing number lines affect player behavior for when they need to design the next game (high \textit{RankErr}). Alternatively, using \gafs would give the designer excellent insight into how different number lines affect players; however, if some 
conditions are 
too difficult, 
we could have caused many players to quit. 
 \forcing provides a useful feedback to the designer without compromising the players' experience (the \textit{RankErr} is close to \gafs but the reward is higher). This is more evident when moving to $w=0.95$, where \forcing significantly improves the reward w.r.t.\ \gafs without losing much accuracy in ranking the arms.

%


\vspace{-0.1in}
\section{Conclusions}\label{s:conclusions}
\vspace{-0.1in}

We studied the tradeoff between rewards and estimation errors. We proposed a new formulation of the problem, introduced a variant of a forced-sampling algorithm, derived bounds on its regret, and we validated our results on synthetic and educational data.

There are a number of possible directions for future work.  \textbf{1)} An active exploration strategy tends to pull all arms a linear \todomout{check this} fraction of time while minimizing regret requires selecting sub-optimal arms a sublinear number of times. It would be interesting to prove an explicit incompatibility result between maximizing $\rho(\blambda)$ and minimizing $\varepsilon(\blambda)$ similar to the result of~\citet{bubeck2009pure} for simple and cumulative regret.
\textbf{2)} While a straightforward application of the \ucb fails, alternative formulations, such as using upper-bounds on both means and variances, could overcome the limitations of \naive. Nonetheless, the resulting function $f_w(\cdot; \{\wt{\nu}_{i,n}\})$ is neither an upper nor a lower bound on the true function $f_w(\cdot; \{\nu_i\})$ and the regret analysis could be considerably more difficult than for \forcing. Furthermore, it would be interesting to study how a Thompson sampling approach could be adapted.
\textbf{3)} Finally, alternative tradeoffs can be formulated (e.g., simple vs.\@ cumulative regret). Notice that the current model, algorithm, and analysis could be easily extended to any strongly-convex and smooth function defined over some parameters of the arms' distributions.


\paragraph{Acknowledgements}
\label{sec:Acknowledgements}
The research presented in this paper was supported by French Ministry of
Higher Education and Research, Nord-Pas-de-Calais Regional Council, Inria and Carnegie Mellon University associated-team project EduBand, and French National Research Agency projects ExTra-Learn (n.ANR-14-CE24-0010-01) and BoB (n.ANR-16-CE23-0003). 

\newpage
\bibliography{ref}
\bibliographystyle{plainnat}

\appendix
\newpage
\onecolumn

\onecolumn
\section{Technical Lemmas}\label{s:proofs}

\begin{proof}[Proof of Lemma~\ref{lem:strong.concave}]
We study the function $g(\cdot) = -f_w(\cdot)$. For any $i=1,\ldots,K$ we have
\begin{align*}
\frac{\partial g(\blambda)}{\partial \lambda_i} = -w\mu_i - \frac{1-w}{2K} \frac{\sigma_i}{\lambda_i^{3/2}}\CommaBin
\end{align*}
and thus the Hessian of $g$ in $\blambda$ is
\begin{align*}
[H_g]_{i,j} = \frac{\partial^2 g(\blambda)}{\partial \lambda_i \partial \lambda_j} = \begin{cases} 0 & \mbox{if } i\neq j\\
\frac{3(1-w)}{4K}\frac{\sigma_i}{\lambda_i^{5/2}} & \mbox{otherwise} \end{cases},
\end{align*}
which means that we have a diagonal matrix. Thus it is easy to show that for any $\blambda\in\wb{\D}_K$, the Hessian is bounded as $\alpha I_{K} \preceq H_g \preceq \beta I_K $, which implies that the function $g$ is $\alpha$-strongly convex and $\beta$-smooth with parameters
\begin{align*}
\alpha &= \frac{3(1-w)\sigma_{\min}}{4K}\CommaBin\\
\beta &= \frac{3(1-w)\sigma_{\max}}{4K\lambda_{\min}^{5/2}}\cdot
\end{align*}
\end{proof}

\begin{proof}[Proof of Lemma~\ref{lem:error.distro}]
The statement is obtained by the series of inequalities
\begin{align*}
\Big|f(\blambda; \{\nu_i\}) &- f(\blambda; \{\wh{\nu}_{i}\})\Big| =\Bigg| w\sum_{i=1}^K \lambda_i(\mu_i - \wh{\mu}_{i}) + \frac{(1-w)}{K} \sum_{i=1}^K \frac{1}{\sqrt{\lambda_i}}(\wh{\sigma}_{i} - \sigma_i) \Bigg| \\
&\leq w\sum_{i=1}^K \lambda_i\big|\mu_i - \wh{\mu}_{i} \big| + \frac{(1-w)}{K} \sum_{i=1}^K \frac{1}{\sqrt{\lambda_i}}\big|\wh{\sigma}_{i} - \sigma_i\big| \\
&\leq w \max_i \epsilon^\mu_i + \frac{(1-w)}{\min_i \sqrt{\lambda_{i}}} \max_i \epsilon^\sigma_i,
\end{align*}
where in the last step we used that $\sum_i \lambda_i = 1$.
\end{proof}

We also derive a simple corollary of Lemma~\ref{lem:error.distro}, which extends the previous result to any (random) choice of allocation $\blambda\in\wb{\D}_K$.

\begin{corollary}\label{cor:error.distro.ext}
After $n$ steps, let $\{\wh{\nu}_{i,n}\}_i$ be the empirical distributions obtained after pulling each arm $T_{i,n}$ times. If we define $\delta_n = \delta/(4Kn^2(n+1))$, then
\begin{align*}
\mathbb{P}\Bigg[\forall n>0, \forall \blambda\in\wb{\D}_K, \big|f(\blambda; \{\nu_i\}) &- f(\blambda; \{\wh{\nu}_{i}\})\big| \leq \max_i\sqrt{\frac{2K\log(2/\delta_n)}{\lambda_{\min}T_{i,n}}}\Bigg] \geq 1-\delta,
\end{align*}
%
\end{corollary}

\begin{proof}
The proof is identical to the one of Lemma~\ref{lem:error.distro} together with a union bound over a covering of the simplex $\wb{\D}_K$ and Prop.~\ref{prop:confidence.intervals} for the concentration of $\wh\mu$ and $\wh\sigma$. We first notice than any covering of the unrestricted simplex also covers $\wb{\D}_K$. We sketch how to construct an $\epsilon$-cover of a $K$-dimensional simplex. For any integer $n = \lceil 1/\epsilon\rceil$, we can design a discretization $\D_K^{(n)}$ of the simplex defined by any possible (fractional) distribution $\wh{\blambda} = (\lambda_1,\ldots,\lambda_K)$ such that for any $\lambda_i$ there exists an integer $j$ such that $\lambda_i = j/n$. $\D_K^{(n)}$ is then an $\epsilon$ cover in $\ell_\infty$-norm since for any distribution $\blambda\in\D_K$ there exists a distribution $\wh{\blambda}\in\D_K^{(n)}$ such that $||\blambda - \wh{\blambda}||_\infty \leq 1/n \leq \epsilon$. The cardinality of $\D_K^{(n)}$ is (loosely) upper-bounded by $n^K$ ($n$ possible integers for each component $\lambda_i$). Upper-bounding the result of Lemma~\ref{lem:error.distro} by $\max\{ \max_i \epsilon^{\mu}_i; \max_i \epsilon^{\sigma}_i \}/\sqrt{\lambda_{\min}}$ (since we focus on $\blambda$ in the restricted simplex $\wb{\D}_K$) and following standard techniques in statistical learning theory (see e.g., Thm.~4 by~\citet{bousquet2003introduction}), we obtain the final statement.
\end{proof}


\begin{proof}[Proof of Lemma~\ref{lem:error.inversion}]

Let $g(\cdot) = -f_w(\cdot)$. For any pair of allocations $\blambda, \blambda'\in\wb{\D}_K$, from Taylor's theorem there exists an allocation $\blambda''$ such that
\begin{align*}
g(\blambda) = g(\blambda') + \nabla g(\blambda')^\top (\blambda-\blambda') + \frac{1}{2} (\blambda - \blambda')^\top H_g(\blambda'') (\blambda-\blambda').
\end{align*}
Given the bound from Lemma~\ref{lem:strong.concave} (strong convexity) and taking $\blambda'=\blambda^*$ (by Asm.~\ref{asm:simplex}, $\blambda^*\in\wb{\D}_K$) we have
\begin{align*}
g(\blambda) \geq g(\blambda^*) + \nabla g(\blambda^*)^\top (\blambda-\blambda^*) + \frac{\alpha}{2} ||\blambda - \blambda^*||_2^2.
\end{align*}
Since $\blambda^*$ is the optimal allocation over $\wb{\D}_K$, the gradient $\nabla g(\blambda^*)$ in $\blambda^*$ in direction towards $\blambda$ is nonnegative and thus
\begin{align*}
\frac{\alpha}{2} ||\blambda - \blambda^*||_2^2 \leq f(\blambda^*) - f(\blambda).
\end{align*}
Given that $||\blambda - \blambda^*||_\infty \leq ||\blambda - \blambda^*||_2$, we finally obtain
\begin{align*}
\max_{i=1,\ldots,K} |\lambda_i - \lambda_i^*| \leq \sqrt{\frac{2}{\alpha} \big(f(\blambda^*) - f(\blambda)\big)}.
\end{align*}
\end{proof}

\begin{proof}[Proof of Lemma~\ref{lem:error.allocation.opt}]

Let $g(\cdot) = -f_w(\cdot)$. For any pair of allocations $\blambda, \blambda'\in\wb{\D}_K$, from Taylor's theorem there exists an allocation $\blambda''$ such that
\begin{align*}
g(\blambda) = g(\blambda') + \nabla g(\blambda')^\top (\blambda-\blambda') + \frac{1}{2} (\blambda - \blambda')^\top H_g(\blambda'') (\blambda-\blambda').
\end{align*}
Given the bound from Lemma~\ref{lem:strong.concave} (smoothness) and taking $\blambda'=\blambda^*$ (by Asm.~\ref{asm:simplex}, $\blambda^*\in\wb{\D}_K$) we have
\begin{align*}
g(\blambda) \leq g(\blambda^*) + \nabla g(\blambda^*)^\top (\blambda-\blambda^*) + \frac{\beta}{2} ||\blambda - \blambda^*||_2^2.
\end{align*}
Consider the term $\nabla g(\blambda)^\top (\blambda-\blambda^*)$. By convexity of $g$, the gradient of $g$ in $\blambda$ towards the optimum $\blambda^*$ is negative. As a result we get
%
\begin{align*}
g(\blambda) \leq g(\blambda^*) + \big(\nabla g(\blambda^*) - \nabla g(\blambda)\big)^\top (\blambda-\blambda^*) + \frac{\beta}{2} ||\blambda - \blambda^*||_2^2.
\end{align*}
Using Cauchy-Schwarz inequality and the fact that for twice differentiable functions, the boundedness of the Hessian (i.e., the smoothness of function) implies that the gradient of $g$ is Lipschitz with coefficient $\beta$, we obtain
\begin{align*}
g(\blambda) \leq g(\blambda^*) + \beta||\blambda-\blambda^*||_2^2 + \frac{\beta}{2} ||\blambda - \blambda^*||_2^2.
\end{align*}
%
Substituting $g$ with $f$ we obtain the desired statement
\begin{align*}
f_w(\blambda^*; \{\nu_i\}) - f_w(\blambda; \{\nu_i\}) \leq \frac{3\beta}{2}\|\blambda-\blambda^*\|^2.
\end{align*}
\end{proof}

We introduce another useful intermediate lemma that states the quality of the estimated optimal allocation.

\begin{lemma}\label{lem:final.error}
Let $\wh{\nu}_{i,n}$ be the empirical distribution characterized by mean $\wh{\mu}_{i,n}$ and variance $\wh{\sigma}_{i,n}$ estimated using $T_{i,n}$ samples. If  $\wh{\blambda}_n = \arg\max_{\blambda\in\wb{\D}_K} f(\blambda; \{\wh{\nu}_i\})$) is the estimated optimal allocation, then
\begin{align*}
f^* - f(\wh{\blambda}_{i,n}; \{\nu_i\}) \leq \max_i2\sqrt{\frac{2K\log(2/\delta_n)}{\lambda_{\min}T_{i,n}}}\cdot
\end{align*}
with probability $1-\delta$, where $\delta_n = \delta/(4Kn^2(n+1))$.
\end{lemma}

\begin{proof}[Proof of Lemma~\ref{lem:final.error}]
The statement follows from the series of inequalities
\begin{align*}
f^* - &f(\wh{\blambda}_{i,n}; \{\nu_i\}) \\
&= f(\blambda^*; \{\nu_i\}) -  f(\blambda^*; \{\wh{\nu}_{i}\}) +  f(\blambda^*; \{\wh{\nu}_i\}) -f(\wh{\blambda}_{i,n}; \{\wh{\nu}_i\}) + f(\wh{\blambda}_{i,n}; \{\wh{\nu}_i\}) - f(\wh{\blambda}_{i,n}; \{\nu_i\})\\
&\leq 2\sup_{\blambda\in\wb{\D}_K} \big|f(\blambda; \{\nu_i\}) -  f(\blambda; \{\wh{\nu}_{i}\})\big|,
\end{align*}
where the difference between the third and fourth term is upper-bounded by $0$ since $\wh{\blambda}_{i,n}$ is the optimizer of $f(\cdot; \{\wh{\nu}_i\})$. The final statement follows from Corollary~\ref{cor:error.distro.ext}.
\end{proof}

\begin{lemma}\label{lem:exploration}
Let consider a function $h(n) = o(n)$ monotonically increasing with $n$. If the forcing condition is
\begin{align}\label{eq:forcing}
T_{i,n} < h(n) + 1,
\end{align}
then for any $n \geq n_0$
\begin{align}\label{eq:lb.pulls}
T_{i,n} \geq h(n),
\end{align}
with $n_0 = \min\{n: \exists \rho\in\mathbb{N}, n=\rho K+1, \text{ s.t. } \rho \geq h(\rho K) + 1\}$ corresponding to the end of the uniform exploration phase.
\end{lemma}

\begin{proof}
The proof of this lemma generalizes Lemma~11 of~\citet{antos2010active}.

\noindent \textbf{Step 1.}
We consider a step $n$ such that (\ref{eq:lb.pulls}) holds. We recall that since $T_{i,n}$ is an integer, then $T_{i,n} \geq \lceil h(n) \rceil$. We define $\Delta(n)$ as the largest number of steps after $n$ in which (\ref{eq:lb.pulls}) still holds, that is
\begin{align*}
\Delta_n = \max\big\{ \Delta\in\mathbb{N}: T_{i,n+\Delta_n} \geq T_{i,n} \geq \lceil h(n) \rceil \geq h(n+\Delta_n) \big\}.
\end{align*}
If for $n-1$ we have $\lceil h(n-1) \rceil = h(n-1)$, then $\Delta_{n-1} = 0$, since $h(n-1) < h(n)$ by definition of $h(n)$ and we say that $n$ is a \textit{reset} step. We use $\wt{n}_l$ with $l\in\mathbb{N}$ to denote the sequence of all reset steps. We define the $l$-th \textit{phase} as $\calP_l = \{\wt{n}_l, \ldots, \wt{n}_l+\Delta_{\wt{n}_l}\}$ and we notice that if there exists a step $n'\in\calP_l$ such that $T_{i,n'}$ satisfies (\ref{eq:lb.pulls}), then for any other step $n''\in\calP_l$ we have
\begin{align*}
T_{i,n''} \geq T_{i,n'} \stackrel{(a)}{\geq} \lceil h(n') \rceil \stackrel{(b)}{\geq} \lceil h(\wt{n}_l) \rceil \stackrel{(c)}{\geq} h(\wt{n}_l+\Delta_{\wt{n}_l}) \stackrel{(d)}{\geq} h(n'')
\end{align*}
where $(a)$ follows from (\ref{eq:lb.pulls}), $(b)$ holds since $n'\geq n$, $(c)$ by definition of $\Delta_n$, and $(d)$ by the fact that $h(n)$ is monotonically increasing in $n$ and $n+\Delta_n \geq n''$. Finally, we also notice that $\Delta_{\wt{n}_l}$ is an non-decreasing function of~$l$ and thus $\calP_l$ becomes longer and longer over time. At this point, we have that (\ref{eq:lb.pulls}) is consistent within each phase $\calP_l$ and thus we need to show that the forcing exploration guarantees that the condition is also preserved across phases.

\textbf{Step 2.}
We study the initial phase of the algorithm. The forcing condition determines a first phase in which all arms are explored uniformly in an arbitrary (but fixed) order (for sake of simplicity let us consider the natural ordering $\{1,\ldots,K\}$. Let $n = \rho K$ for some $\rho \in \mathbb{N}$ during the uniform exploration phase, then at the beginning of step $n$ arm $K$ is pulled and at the end of the step all arms have $T_{i,n+1} = \rho$ samples. The end of the exploration phase corresponds to the smallest value of $\rho$ so that step $n=\rho K + 1$ is such that $T_{i,n} = \rho \geq h(n) + 1 = h(\rho K) + 1$, so that the forcing condition is not triggered any more. We also notice $T_{i,n} \geq h(\rho K)$ satisfies (\ref{eq:lb.pulls}) and that $n = \rho K +1$ is a reset step (i.e., $\lceil n-1\rceil = \rho K$) and thus we denote by $\wt{n}_1 = \rho K +1$ the beginning of the first phase $\calP_1$ and by step 1, we obtain that for all $n'\in\calP_1$, $T_{i,n} \geq h(\rho K + 1+\Delta_{\rho K})$ (i.e., (\ref{eq:lb.pulls}) keeps holding). This is the base for induction.

\textbf{Step 3.}
We assume that (\ref{eq:lb.pulls}) holds for a step $\wh{n}_l$ at the beginning of phase $\calP_l$ for all arms, then by step 1, (\ref{eq:lb.pulls}) also holds for any other step until $n+\Delta_n$ independently on whether the arms are pulled or not. We study what happens at the beginning of the successive phase starting at $n+\Delta_n+1$. We first consider all arms $i$ for which $T_{i,n} = \lceil h(n) \rceil$, then we have $T_{i,n} < h(n) +1$, which implies that the forcing exploration is triggered on this arm. Since there are potentially $K$ arms in this situation, it may take as long as $K$ steps before updating them all. Thus, if $\Delta_n > K$, then
\begin{align*}
T_{i,\wt{n}_l+\Delta_{\wt{n}_l}+1} = T_{i,\wt{n}_{l+1}} \geq T_{i,\wt{n}_l+K} \stackrel{(a)}{\geq} \lceil h(\wt{n}_l) \rceil + 1 \stackrel{(b)}{\geq} h(\wt{n}_l+\Delta_{\wt{n}_l}) + 1 \stackrel{(c)}{>} h(\wt{n}_l+\Delta_{\wt{n}_l}+1),
\end{align*}
where in $(a)$ we use the fact that $i$ is forced to be pulled, $(b)$ follows from the definition of $\Delta_n$, and $(c)$ is by the sub-linearity of $h(n)$. Then we focus on the arms for which $T_{i,n} \geq \lceil h_n \rceil + 1$. Since $\lceil h_{\wt{n}_l} \rceil + 1 < h_{\wt{n}_l} +1$ the forcing condition is not met (at least at the beginning). We have that even if during in phase $\calP_l$ these arms are never pulled, we have
\begin{align*}
T_{i,\wt{n}_l+\Delta_{\wt{n}_l}+1} = T_{i,\wt{n}_{l+1}} \geq T_{i,\wt{n}_l} \geq \lceil h(\wt{n}_l) \rceil + 1 > h(\wt{n}_l+\Delta_{\wt{n}_l}+1),
\end{align*}
where the arguments are as above. This concludes the inductive step showing that if (\ref{eq:lb.pulls}) holds in a phase $\calP_l$ then it holds at $\calP_{l+1}$ as well as soon as $\Delta_{\wt{n}_l} > K$. As a result, step 2, together with the condition on $n$ for the end of the exploration phase, and step 3 prove the statement.
\end{proof}

\begin{corollary}\label{cor:exploration}
If $h(n) = \eta\sqrt{n}$, then for any $n\geq n_0 = K(K\eta^2 + \eta\sqrt{K} + 1)$ and all arms $T_{i,n} \geq \eta\sqrt{n}$.
\end{corollary}

\begin{proof}
We just need to derive the length of the exploration phase $n_0 = \min\{n: \exists \rho\in\mathbb{N}, n=\rho K+1, \text{ s.t. } \rho \geq h(\rho K) + 1\}$
\begin{align*}
\rho \geq \eta \sqrt{\rho K} + 1.
\end{align*}
Solving for $\rho$ and upper-bounding the condition gives $\rho \geq \eta^2 K + \eta\sqrt{K} + 1$, which gives $n_0 =K(K\eta^2 + \eta\sqrt{K} + 1)$.
\end{proof}

\begin{lemma}\label{lem:tracking}
We assume that there exists a value $n_1$ after which $\wh{\blambda}_n$ is constantly a good approximation of $\blambda^*$, i.e., there exists a monotonically  decreasing function $\omega(n)$ such that for any step $n \geq n_1$
\begin{align*}
\max_{i=1,\ldots,K} \big|\wh{\lambda}_{i,n} - \lambda_i^* \big| \leq \omega(n).
\end{align*}
Furthermore, we assume that for any $n\geq n_1$ the following condition holds
\begin{align}\label{eq:condition.omega}
2n\omega(n) \geq \eta\sqrt{n} + 1.
\end{align}
Then for any arm $i$
\begin{align*}
-(K-1) \max\bigg\{ \frac{n_1}{n}, 2\omega(n)+1\bigg\} \leq \wt{\lambda}_{i,n} - \lambda_i^* \leq \max\bigg\{ \frac{n_1}{n}(1-\lambda_i), 2\omega(n)+1\bigg\}.
\end{align*}
\end{lemma}

\begin{proof}
This lemma follows from similar arguments as Lemma~4 by~\citet{antos2010active}. Nonetheless, given the use of a slightly different tracking rule, we provide the full proof here.

We study the error $\eps_{i,n} = \wt{\lambda}_{i,n} - \lambda_i^*$. Since $T_{i,n+1} = T_{i,n} + \mathbb{I}\{I_n = i\}$, we have
\begin{align*}
\eps_{i,n+1} &= \frac{T_{i,n} + \mathbb{I}\{I_n = i\}}{n+1} - \frac{n+1}{n+1}\lambda^*_i\\
&= \frac{n}{n+1}\bigg( \frac{T_{i,n}}{n} - \lambda^*_i\bigg) + \frac{\mathbb{I}\{I_n = i\} - \lambda_i^*}{n+1}\\
&= \frac{n}{n+1}\eps_{i,n} + \frac{\mathbb{I}\{I_n = i\} - \lambda_i^*}{n+1}\cdot
\end{align*}
Then we need to study the arm selection rule at step $n$ to understand the evolution of the error and its relationship with the error of $\wh{\blambda}$. We have
\begin{align*}
\mathbb{I}\{I_n = i\} \leq \mathbb{I}\{T_{i,n} < \eta\sqrt{n}+1 \text{ or } i= \arg\min_j (\wt{\lambda}_{i,n} - \wh{\lambda}_{i,n})\}.
\end{align*}
We study the tracking condition. Let $i= \arg\min_j (\wt{\lambda}_{j,n} - \wh{\lambda}_{j,n})$ then
\begin{align*}
\wt{\lambda}_{i,n} &= \wh{\lambda}_{i,n} + \min_j \big(\wt{\lambda}_{j,n} - \wh{\lambda}_{j,n}\big) \\
&= \lambda_i^* + \wh{\lambda}_{i,n} - \lambda_i^* + \min_j \big(\wt{\lambda}_{j,n} - \lambda_j^* + \lambda_j^* - \wh{\lambda}_{j,n}\big) \\
&\leq \lambda_i^* + \wh{\lambda}_{i,n} - \lambda_i^* + \min_j \big(\wt{\lambda}_{j,n} - \lambda_j^*\big) + \max_j\big(\lambda_j^* - \wh{\lambda}_{j,n}\big) \\
&\leq \lambda_i^* + 2\max_j \big|\lambda_j^* - \wh{\lambda}_{j,n}\big|\\
&\leq \lambda_i^* + 2 \omega(n),
\end{align*}
where we use the fact that since $\sum_j \big(\wt{\lambda}_{j,n} - \lambda_j^*\big) = 0$ and all proportions are bigger than zero, then the minimum over $\big(\wt{\lambda}_{j,n} - \lambda_j^*\big)$ is nonpositive. We now study the forcing condition $T_{i,n} = n \wt{\lambda}_{i,n} < \eta\sqrt{n} + 1$. For $n \geq n_1$, we have
\begin{align*}
\wt{\lambda}_{i,n} - \lambda_i^* \leq \frac{\eta}{\sqrt{n}} + \frac{1}{n} \leq 2 \omega(n),
\end{align*}
where the last step follows from the properties of $\omega$ defined in the statement. Then we can simplify the condition under which arm $i$ is selected as
\begin{align*}
\mathbb{I}\{I_n = i\} \leq \mathbb{I}\{\eps_{i,n} \leq 2\omega(n)\}.
\end{align*}
We proceed by defining $E_{i,n} = n\eps_{i,n}$ and the corresponding process $E_{i,n+1} = E_{i,n} + \mathbb{I}\{I_n = i\} - \lambda_i^*$. We also introduce
\begin{align*}
\wt{E}_{i,n_1} &= E_{i,n_1},\\
\wt{E}_{i,n+1} &= \wt{E}_{i,n} + \mathbb{I}\{\wt{E}_{i,n} \leq 2n\omega(n)\} - \lambda_i^*,
\end{align*}
which follows the same dynamics of $E_{i,n}$ except for the fact that the looser arm selection is considered. From Lemma~5 of~\citet{antos2010active}, we have $E_{i,n} \leq \wt{E}_{i,n}$ for any $n\geq n_1$. It is easy to see that $\wt{E}_{i,n}$ satisfies the following inequality
\begin{align*}
\wt{E}_{i,n} \leq \max\{ E_{k,n_1}, 2n\omega(n)+1\} \leq \max\{ n_1, 2n\omega(n)+1\},
\end{align*}
where the last inequality follows from the fact that $E_{k,n_1} \leq n_1(1-\lambda_i)$. Then, we obtain the upper-bound
\begin{align*}
\eps_{i,n} \leq \max\bigg\{ \frac{n_1}{n}(1-\lambda_i), 2\omega(n)+1\bigg\}.
\end{align*}
From the upper-bound, we obtain a lower bound on $\eps_{i,n}$ as
\begin{align*}
\eps_{i,n} = -\sum_{j\neq i} \eps_{i,n} \geq -(K-1)\max_j \eps_{j,n} \geq -(K-1) \max\bigg\{ \frac{n_1}{n}, 2\omega(n)+1\bigg\},
\end{align*}
which concludes the proof.
\end{proof}


\section{Proof of Theorem~\ref{thm:regret}}\label{s:thm.proof}

In this section, we report the full proof of the main regret theorem, whose complete statement is as follows.

\textbf{Theorem~\ref{thm:regret}}
\textit{
We consider a MAB with $K \geq 2$ arms characterized by distributions $\{\nu_i\}$ with mean $\{\mu_i\}$ and variance $\{\var_i\}$. Consider \forcing with a parameter $\eta\leq 21$ and a simplex restricted to $\lambda_{\min}$. Given a tradeoff parameter $w$ and under Asm.~\ref{asm:simplex}, \forcing suffers a regret
\begin{align*}
R_n(\wt{\blambda}) \leq
\begin{cases}
1 & \mbox{if } n \leq n_0 \\
\mathlarger{43 K^{5/2} \frac{\beta}{\alpha} \sqrt{\frac{\log(2/\delta_n)}{\eta\lambda_{\min}}} n^{-1/4}} &\mbox{if } n_0 < n \leq n_2\\
\mathlarger{153 K^{5/2} \frac{\beta}{\alpha} \sqrt{\frac{\log(2/\delta_n)}{\lambda_{\min}\lambda_{\min}^*}}}n^{-1/2} & \mbox{if } n > n_2,
\end{cases}
\end{align*}
with probability $1-\delta$ (where $\delta_n =  \delta/(4Kn(n+1))$) and
\begin{align*}
n_0 = K(K\eta^2 + \eta\sqrt{K} + 1), \enspace \text{ and } \enspace n_2 = \frac{C}{(\lambda_{\min}^*)^8} \frac{1}{\alpha^4\lambda_{\min}^4} K^{10} \log(1/\delta_n)^{2},
\end{align*}
where $C$ is a suitable numerical constant.
}


\textbf{Sketch of the proof.} In the active exploration problem, \citet{antos2010active} rely on the fact that minimizing $\varepsilon(\blambda)$ has a closed-form solution w.r.t.\ the parameters of the problem (i.e., the variance of the arms) and errors in estimating the parameters are directly translated into deviations between $\wh{\blambda}_{i,n}$ and $\blambda^*$. In our case, $\blambda^*$ has no closed-form solution and we need to explicitly ``translate'' errors in estimating $f_w$ into the deviations between allocations and vice versa. Furthermore, \forcing uses a slightly different tracking strategy (see Lemma~\ref{lem:tracking}) and we prove the effect of the forcing exploration for a more general condition (see Lemma~\ref{lem:exploration}). The proof follows the following steps. We first exploit the forcing exploration of the algorithm to guarantee that each arm is pulled at least $\wt{O}(\sqrt{n})$ at each step $n$. Through Prop.~\ref{prop:confidence.intervals}, Lemma~\ref{lem:error.distro}, and Lemma~\ref{lem:error.inversion} we obtain that the allocation $\wh{\blambda}_{n}$ converges to $\blambda^*$ with a rate $\wt{O}(1/n^{1/8})$. We show that the tracking step is executed often enough (w.r.t.\ the forced exploration) and it is efficient enough to propagate the errors of $\wh{\blambda}_n$ to the actual allocation $\wt{\blambda}_{n}$, which also approaches $\blambda^*$ with a rate $\wt{O}(1/n^{1/8})$. Unfortunately, this does not translate in a satisfactory regret bound but it shows that for $n$ \textit{big enough}, $T_{i,n}$ is only a fraction away from the desired number of pulls $n\lambda^*_i$, which provides a more refined lower-bound on $T_{i,n} = \wt{\Omega}(n)$. In this second phase ($n \geq n_2$), the estimates $\wh{\nu}_{i,n}$ are much more accurate (Prop.~\ref{prop:confidence.intervals}) and through Lemma~\ref{lem:error.distro} and Lemma~\ref{lem:error.inversion}, the accuracy of $\wh{\blambda}$ and $\wt{\blambda}$ improves to $\wt{O}(1/n^{1/4})$. At this point, we apply Lemma~\ref{lem:error.allocation.opt} and translate the guarantee on $\wt{\blambda}$ to the final regret bound. While in Prop.~\ref{prop:confidence.intervals} we use simple Chernoff-Heoffding bounds, Bernstein bounds (which consider the impact of the variance on the concentration inequality) could significantly improve the final result. While the asymptotic rate would remain unaffected, we conjecture that this more refined analysis would show that whenever arms have very small variance the regret $R_n$ decreases as $\wt{O}(1/n)$ before converging to the asymptotic rate.

We now proceed with the formal proof. We start with a technical lemma.

\begin{lemma}\label{lem:n1}
If $\eta \leq 21$, then
\begin{align*}
2n\sqrt{\frac{20K}{\alpha\lambda_{\min}}} \bigg(K\frac{\log(1/\delta_n)}{2\eta\sqrt{n}}\bigg)^{1/4} \geq \eta\sqrt{n} + 1,
\end{align*}
for any $n \geq 4$.
\end{lemma}

\begin{proof}
We study the value of $n_1$ to satisfy Eq.~\ref{eq:condition.omega} for $\omega(n)$, which is defined later on in step 5 of the final proof. This corresponds to finding the minimal value of $n\in\mathbb{N}$ such that
\begin{align*}
2n\sqrt{\frac{20K}{\alpha\lambda_{\min}}} \bigg(K\frac{\log(1/\delta_n)}{2\eta\sqrt{n}}\bigg)^{1/4} \geq \eta\sqrt{n} + 1.
\end{align*}
We proceed by successive (often loose) simplifications to the previous expression. Since $K \geq 2$ and $n \geq 1$, we have that $\delta_n \leq \delta/16$. If we choose $\delta < 1/2$, then $\delta_n \leq 1/32$ and $\log(1/\delta_n) > 3$. As a result, we obtain that the previous condition can be written as
\begin{align*}
&2n\sqrt{\frac{40}{\alpha\lambda_{\min}}} \bigg(\frac{3}{\eta\sqrt{n}}\bigg)^{1/4} \geq \eta\sqrt{n} + 1\\
&\Rightarrow 16n \frac{1}{\sqrt{\alpha\lambda_{\min}}} \bigg(\frac{1}{\eta\sqrt{n}}\bigg)^{1/4} \geq \eta\sqrt{n} + 1\\
&\Rightarrow 16n \frac{1}{\sqrt{\alpha\lambda_{\min}}} - \eta^{5/4}n^{5/8} - \eta^{1/4}n^{1/8} \geq 0.
\end{align*}
We further study the first multiplicative term and use the fact that $\lambda_{\min} \leq 1/K \leq 1/2$, $\alpha = 2(1-w)\sigma_{\min}^2/K \leq 1/4$, then
\begin{align*}
&2n\omega(n) \geq \eta\sqrt{n} + 1\\
&\Rightarrow 16n \frac{1}{\sqrt{1/8}} - \eta^{5/4}n^{5/8} - \eta^{1/4}n^{1/8} \geq 0\\
&\Rightarrow 45n - \eta^{5/4}n^{5/8} - \eta^{1/4}n^{1/8} \geq 0.
\end{align*}
Using the condition that $45 \geq \eta^{5/4}$ and the fact that $\eta^{5/4} \leq \eta^{1/4}$, we can further simplify the last term as
\begin{align*}
&2n\omega(n) \geq \eta\sqrt{n} + 1\\
&\Rightarrow n - n^{5/8} - n^{1/8} \geq 0,
\end{align*}
which is satisfied for any $n \geq 4$.
\end{proof}

\begin{proof}[Proof of Theorem~\ref{thm:regret}]

\noindent \textbf{Step 1 (accuracy of empirical estimates).}
In Alg.~\ref{alg:forcing} we explicitly force all arms to be pulled a minimum number of times. In particular, from Lemma~\ref{lem:exploration} we have that for any $n \geq n_0 = K(K\eta^2 + \eta\sqrt{K} + 1)$ then $T_{i,n} \geq \eta\sqrt{n}$. From Proposition~\ref{prop:confidence.intervals}, if $\delta_n = \delta/(4Kn^2(n+1))$, then for any arm $i$ we have\footnote{Here we already use $\delta_n$ and form of the confidence intervals used in Corollary~\ref{cor:error.distro.ext}.}
\begin{align*}
\big| \wh{\mu}_{i,n} - \mu_i \big| &\leq \sqrt{\frac{\log(1/\delta_n)}{2\eta\sqrt{n}}} \\
\big| \wh{\sigma}^2_{i,n} - \sigma^2_i \big| &\leq \sqrt{\frac{2K\log(2/\delta_n)}{2\eta\sqrt{n}}},
\end{align*}
with probability at least $1-\delta$.

\noindent \textbf{Step 2 (accuracy of function estimate).}
Using Corollary~\ref{cor:error.distro.ext} we can bound the error on the function $f$ when using the estimates $\wh{\nu}_{i,n}$ instead of the true parameters $\nu_i$. In fact, for any $\blambda\in\wb{\D}_K$ we have
\begin{align*}
\big|f(\blambda; \{\nu_i\}) - f(\blambda; \wh{\nu}_{i,n})\big| &\leq \sqrt{\frac{2K\log(2/\delta_n)}{\lambda_{\min}\eta\sqrt{n}}}
\end{align*}
\noindent \textbf{Step 3 (performance of estimated optimal allocation).}
We can derive the performance of the allocation $\wh{\blambda}_{n}$ computed on the basis of the estimates obtained after $n$ samples. From Lemma~\ref{lem:final.error} we have
\begin{align*}
\big|f(\blambda^*; \{\nu_i\}) - f(\wh{\blambda}_{n}; \{\nu_i\})\big| \leq 2\sqrt{\frac{2K\log(2/\delta_n)}{\lambda_{\min}\eta\sqrt{n}}}.
\end{align*}
\noindent \textbf{Step 4 (from performance to allocation).}
From Lemma~\ref{lem:error.inversion} we have that a loss in performance in terms of the function $f$ implies the similarity of the estimated allocation to $\blambda^*$. For any arm $i=1,\ldots,K$ we have
\begin{align*}
|\wh{\lambda}_{i,n} - \lambda^*_i| \leq \sqrt{\frac{4}{\alpha}} \bigg(2K\frac{\log(2/\delta_n)}{\lambda_{\min}\eta\sqrt{n}}\bigg)^{1/4}.
\end{align*}
\noindent \textbf{Step 5 (tracking).}
The algorithm is designed so that the actual allocation $\wt{\blambda}_{n}$ (i.e., the fraction of pulls allocated to each arm until step $n$) is \textit{tracking} the optimal estimated allocation $\wh{\blambda}_n$. Since the difference between $\wh{\blambda}_{n}$ and $\blambda^*$ is bounded as above, we can use Lemma~\ref{lem:tracking} to bound the accuracy of $\wt{\blambda}$. We define
\begin{align*}
\omega(n) := \sqrt{\frac{4}{\alpha}} \bigg(2K\frac{\log(2/\delta_n)}{\lambda_{\min}\eta\sqrt{n}}\bigg)^{1/4},
\end{align*}
then from Lemma~\ref{lem:n1} we have that for any $n\geq 4$ (and $\eta \leq 21$), the condition in Eq.~\ref{eq:condition.omega} ($2n\omega(n) \geq \eta\sqrt{n} + 1$) is satisfied and we can apply Lemma~\ref{lem:tracking}. In particular, we have that $n_1 = \max\{5, n_0\} \leq n_0$ guarantees both conditions in the lemma and this implies that for any arm $i=1,\ldots,K$
\begin{align}
\label{eq:yolo} 
|\wt{\lambda}_{i,n} - \lambda^*_i| \leq \eta(n) := (K-1)\max\bigg\{ \frac{n_0}{n}; 2\omega(n) + 1 \bigg\}.
\end{align}
If we stopped at this point, the regret could be bounded using Lemma~\ref{lem:error.allocation.opt} as
\begin{align*}
f(\blambda^*; \{\nu_i\}) - f(\wt{\blambda}_{n}; \{\nu_i\}) \leq 34 K^{5/2} \frac{\beta}{\alpha} \sqrt{\frac{\log(1/\delta_n)}{\eta\lambda_{\min}}} n^{-1/4},
\end{align*}
which is decreasing to zero very slowly.

\noindent \textbf{Step 6 (linear pulls).}
From Eq.~\ref{eq:yolo}, we can than easily derive a much stronger guarantee on the number of pulls allocated to any arm $i$. Let $n_2 = \min\{n\in\mathbb{N}: \eta(n) \leq \min_i\lambda^*_{i}/2\}$, then for any $n \geq n_2$ we have
\begin{align*}
|\wt{\lambda}_{i,n} - \lambda^*_i| \leq \lambda^*_i/2,
\end{align*}
which implies that
\begin{align*}
T_{i,n} \geq n\lambda_{i}^*/2.
\end{align*}
\noindent \textbf{Step 7 (regret bound).}
At this point we can reproduce the steps~1, 2, and~3 using $T_{i,n} \geq n\lambda_i^*/2 \geq n\lambda_{\min}^*/2$ samples and we obtain that for any $n\geq n_2$
\begin{align*}
\big|f(\blambda^*; \{\nu_i\}) - f(\wh{\blambda}_{n}; \{\nu_i\})\big| \leq 4\sqrt{\frac{K\log(2/\delta_n)}{n\lambda_{\min}^*\lambda_{\min}}}\cdot
\end{align*}
Unfortunately this guarantee on the performance of the optimal estimated allocation does not directly translate into a regret bound on $\wt{\blambda}_n$ (i.e., the actual distribution implemented by the algorithm up to step $n$). We first apply the same idea as in step 4 and obtain
\begin{align*}
|\wh{\lambda}_{i,n} - \lambda^*_i| \leq \omega'(n) := \sqrt{\frac{8}{\alpha}} \bigg(K\frac{\log(2/\delta_n)}{2n\lambda_{\min}^*\lambda_{\min}}\bigg)^{1/4}.
\end{align*}
By applying a similar argument as in Lemma~\ref{lem:n1}, we obtain that $n_1 = \max\{4, n_0\} = n_0$ and the tracking argument in step~5 gives us
\begin{align*}
|\wt{\lambda}_{i,n} - \lambda^*_i| \leq \eta'(n) := (K-1)\max\bigg\{ \frac{n_0}{n}; 2\omega'(n) + 1 \bigg\}.
\end{align*}
At this point we just need to apply Lemma~\ref{lem:error.allocation.opt} on the difference between $\wt{\blambda}$ and $\blambda^*$ and obtain
\begin{align*}
f(\blambda^*; \{\nu_i\}) - f(\wt{\blambda}_{n}; \{\nu_i\}) \leq \frac{3\beta}{2} ||\wt{\blambda}_{n} - \blambda^*_i||_2^2 \leq \frac{3\beta K^{2}}{2} \max\bigg\{ \frac{n_0^2}{n^2}; 9(\omega'(n))^2 \bigg\},
\end{align*}
where we used $||\wt{\blambda}_{n} - \blambda^*_i||_2 \leq \sqrt{K} \max_i |\wt{\lambda}_{i,n} - \lambda^*_i|$.
Using the definition of $\omega'(n)$ gives the final statement
\begin{align*}
f(\blambda^*; \{\nu_i\}) - f(\wt{\blambda}_{n}; \{\nu_i\}) \leq 153K^{5/2}\frac{\beta}{\alpha}  \sqrt{\frac{\log(2/\delta_n)}{n\lambda^*_{\min}\lambda_{\min}}}\cdot
\end{align*}

\noindent \textbf{Step 8 (condition on $n$).}
From the definition of $n_2$, we have that $n_2$ is at most a value $n$ such that $\eta(n) \leq \lambda_i^*/2$. We consider the worst case form $\lambda_i^*$, that is $\lambda_{\min}^*$ and we bound separately the two possible terms in the $\max$ in the definition of $\eta(n)$. The first term should satisfy
\begin{align*}
(K-1) \frac{n_0}{n} \leq \frac{\lambda_i^*}{2} \Rightarrow n \geq \frac{2n_0(K-1)}{\lambda_{\min}^*}\cdot
\end{align*}
For the second term we have a much slower decreasing function $\omega(n)$ and the condition is
\begin{align*}
&K\sqrt{\frac{4}{\alpha}} \bigg(2K\frac{\log(2/\delta_n)}{\lambda_{\min}\eta\sqrt{n}}\bigg)^{1/4}\!\!\!\! \leq \frac{\lambda_i^*}{2} \\
&\Rightarrow n^{1/8}  \geq \frac{4}{\lambda_{\min}^*} \frac{K^{5/4}}{\sqrt{\alpha}}  \bigg(\frac{\log(2/\delta_n)}{\lambda_{\min}}\bigg)^{1/4},
\end{align*}
which is clearly more constraining than the first one. Then we define
\begin{align*}
n_2 = \frac{4^8}{(\lambda_{\min}^*)^8} \frac{K^{10}}{\alpha^4} \frac{\log^2(1/\delta_n)}{\lambda_{\min}^2}\cdot
\end{align*}

\end{proof}


\section{Supplementary Results}\label{app:experiments}

\begin{table}
\centering
\begin{tabular}{| c | c | c | c |}
\hline
 Exp.&$\blambda^*$ &$\mu$ & $\sigma^2$\\ \hline
 \hline
1&$0.57$  {\footnotesize (balanced)}  & (1.5,1) & (1,1) \\
2&$0.56$  {\footnotesize (balanced)} & (2,1) & (1,2) \\
3&$0.28$ {\footnotesize (unbalanced)} & (1.1, 1) & (0.1, 2) \\
4&$0.85$  {\footnotesize (unbalanced)}  & (3,1) & (0.1,0.1) \\
\hline
\end{tabular}
\caption{Optimal allocation for a MAB with $K=2$ arms, different values of mean and variance, and $w=0.4$.}
\label{tab:exp.setting.k2}
\end{table}

\textbf{Optimal allocation.}
We consider four different MABs with $K=2$ arms with means and variances reported in Table~\ref{tab:exp.setting.k2} and $w=0.4$ (i.e., slightly more weight to the estimation errors). In the first two settings, we notice that $\blambda^*$ is an almost balanced allocation over the two arms, with a slight preference for arm $1$. In the first case, this is due to the fact that the variance of the arms is exactly the same, which suggests a perfectly even allocation would guarantee equal estimation accuracy of the means. Nonetheless, since arm $1$ has a larger mean, this moves $\blambda^*$ towards it. On the other hand, in the second case both means and variances are unbalanced, but while $\sigma_2^2 > \sigma_1^2$ suggests that arm $2$ should be pulled more (to compensate for the larger variance), $\mu_1 \gg \mu_2$ requires arm $1$ to be pulled much more than arm $1$. As a result, $\blambda^*$ is still balanced. In the third and fourth setting, $\blambda^*$ recommends selecting one arm much more than the other. While in the third setting this is due to a strong unbalance in the variances, in the fourth setting this is induced by the difference in the mean. 

\textbf{Comparison with Naive-UCB.}
Before reporting empirical results on the straightforward UCB-like algorithm (called \naive) illustrated in Sect.~\ref{s:algorithms}, where an upper-bound on the function $f_w$ is constructed at each step, we first provide a preliminary example. Consider the case (very extreme for illustrative purposes) $w=0$, $\sigma_1=1$, $\sigma_2=2$, for which $\blambda*\approx[0.38 , 0.62]$. Assume that after pulling each arm twice, we have $\wh{\sigma}_1=2, \wh{\sigma}_2=0.1$. Using \naive, the estimated optimal allocation would be very close to $[1 0]$ (i.e., only arm 1 is pulled). Then \naive keeps selecting arm 1, while arm2, whose lower-bound on the variance remains very small ($1/\sqrt{T_2}$ is large), is almost never pulled, thus preventing the estimate from converging and the algorithm to have a small regret. This shows the algorithm has a constant regret with a fixed probability. 

In Fig.~\ref{fig:ucb.naive} we report the rescaled regret $\wt{R}_n = \sqrt{n} R_n$ for both \forcing and \naive. More in detail, we define an upper-bound of $f_w$ as in Eq.~\ref{eq:naive.ucb}, but we threshold the lower bounds on the variance to a constant (0.01 in the experiments). Then we compute $\wh{\blambda}_n$ as the optimal allocation of $f_w^{UB}$ and the same tracking arm selection as in \forcing is used. On the other hand, \naive does not use any forced exploration, while it relies on the optimism-of-uncertainty principle to sufficiently explore all arms. The comparison is reported for settings 2 and 3 of Table~\ref{tab:exp.setting.k2}.

For both settings, \forcing performs as well as expected and its rescaled regret eventually converges to a constant (in this case the constant is very small since we have only two arms). On the other hand, \naive achieves very contrasting results. In setting 2 (Fig.~\ref{fig:ucb.naive}-\textit{left}), in a first phase \naive suffers a regret with a slower rate than $\wt{O}\left(1/\sqrt{n}\right)$ since the rescaled regret is increasing. While in \forcing this phase is limited to the initial exploration of all arms, in \naive this is due to the fact that the algorithm is underestimating the variances of the arms and it tends to be more aggressive in selecting arms with larger mean (arm $1$ in this case), which corresponds to very limited exploration to the other arm. The only residual sources of exploration are triggered by the fact that lower bound are capped to a small constant (when negative), which encourages partial exploration, and upper-confidence bounds on the means, which induce a \ucb-like strategy where the two arms are explored to identify the best. As a result, there is a long phase of poor performance, until enough exploration is achieved to have estimates which allow an accurate estimate of $\blambda^*$ and thus a regret which decreases again as $\mathcal{O}\left(1/\sqrt{n}\right)$.
In setting 3 (Fig.~\ref{fig:ucb.naive}-\textit{right}), the optimal allocation is very biased towards the second arm (see Table \ref{tab:exp.setting.k2}). However, \naive would target more the first arm attracted by its high mean upper bound (i.e., $\wt{\lambda}_1 \gg \wt{\lambda}_2$ in contrast with $\blambda^*$). As a result, its regret is much higher in this case than in the previous case. Actually, with the horizon of $n=5000$ the regret is constant (and thus the rescaled regret increases as $\wt{O}(\sqrt{n})$) and the algorithm does not seem to be able to recover from bad estimates of the variance.

\begin{figure}
\centering
\includegraphics[width=0.45\columnwidth]{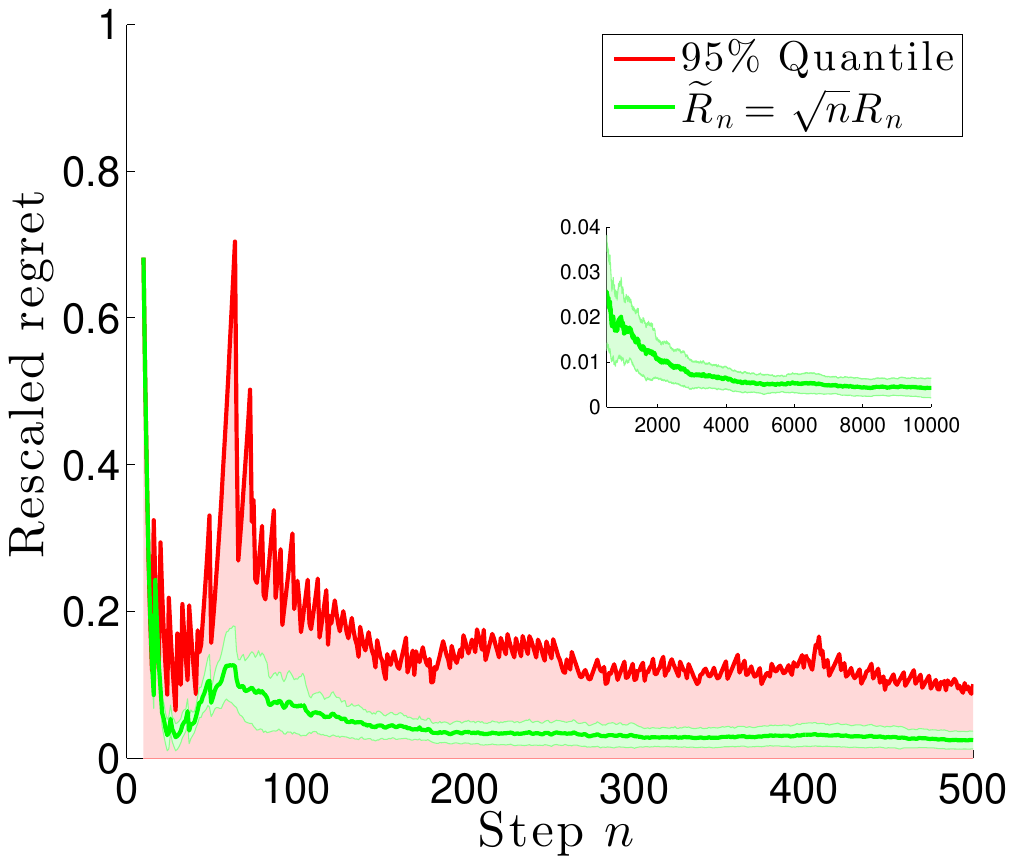}
\includegraphics[width=0.45\columnwidth]{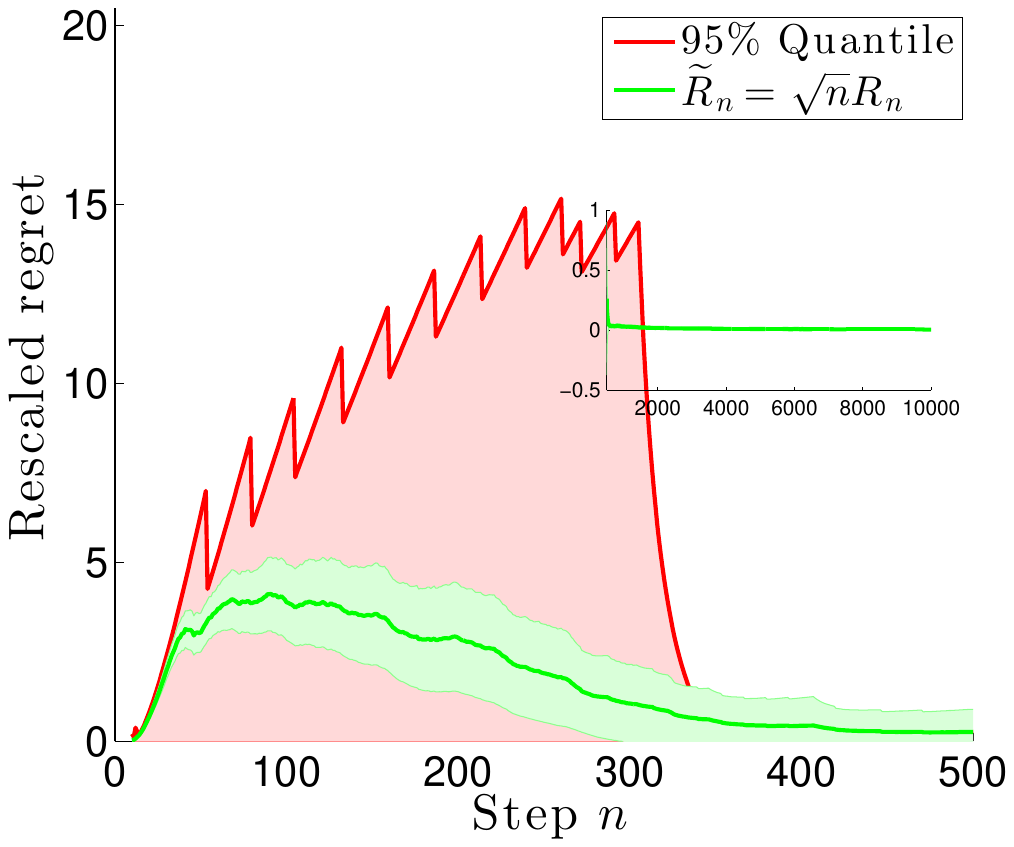}
\includegraphics[width=0.45\columnwidth]{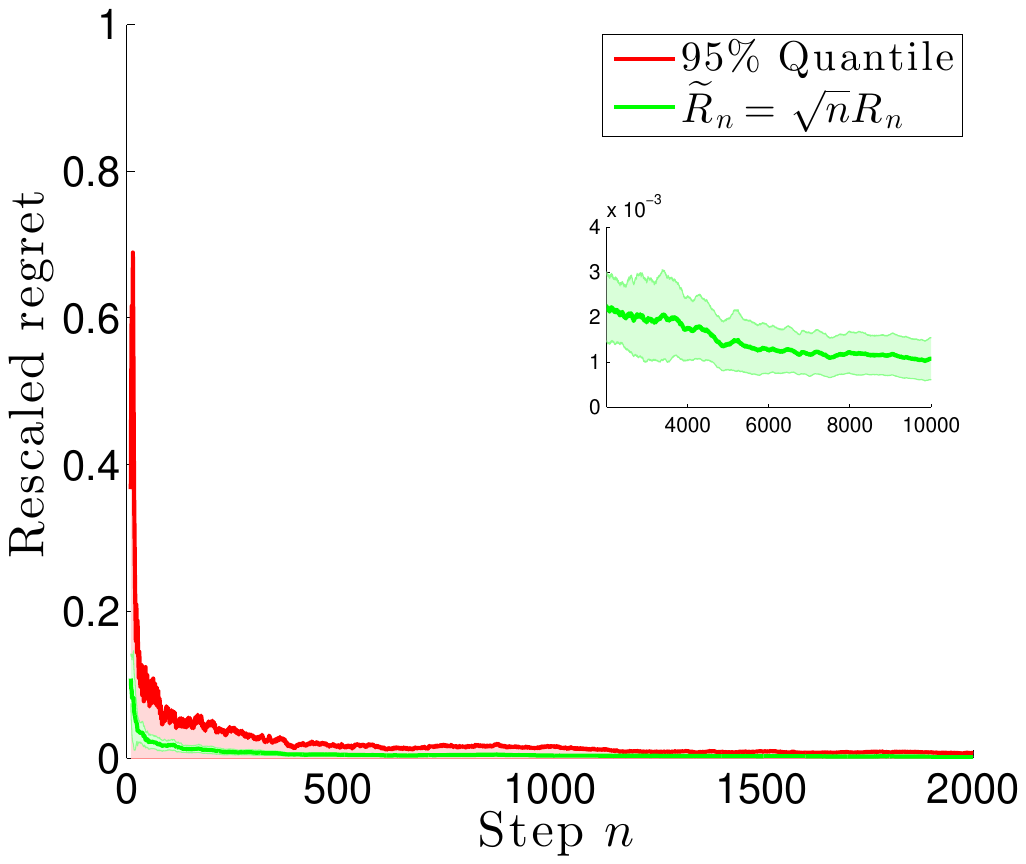}
\includegraphics[width=0.45\columnwidth]{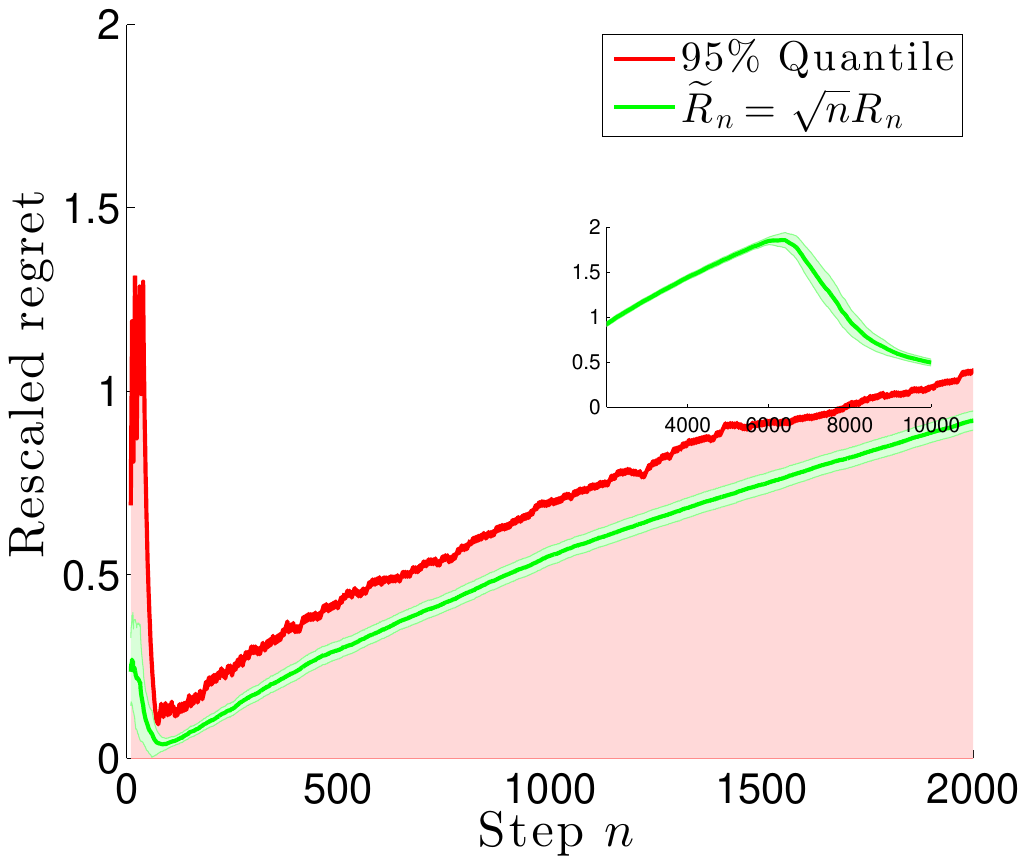}
\caption{Rescaled regret for \forcing and \naive on settings 2 and 3 of Table~\ref{tab:exp.setting.k2}. Notice the difference in the scale of the $x$ and $y$ axes.}
\label{fig:ucb.naive}
\end{figure}
\begin{figure}
\centering
\includegraphics[trim={1cm 0 0cm 0},width=0.32\columnwidth]{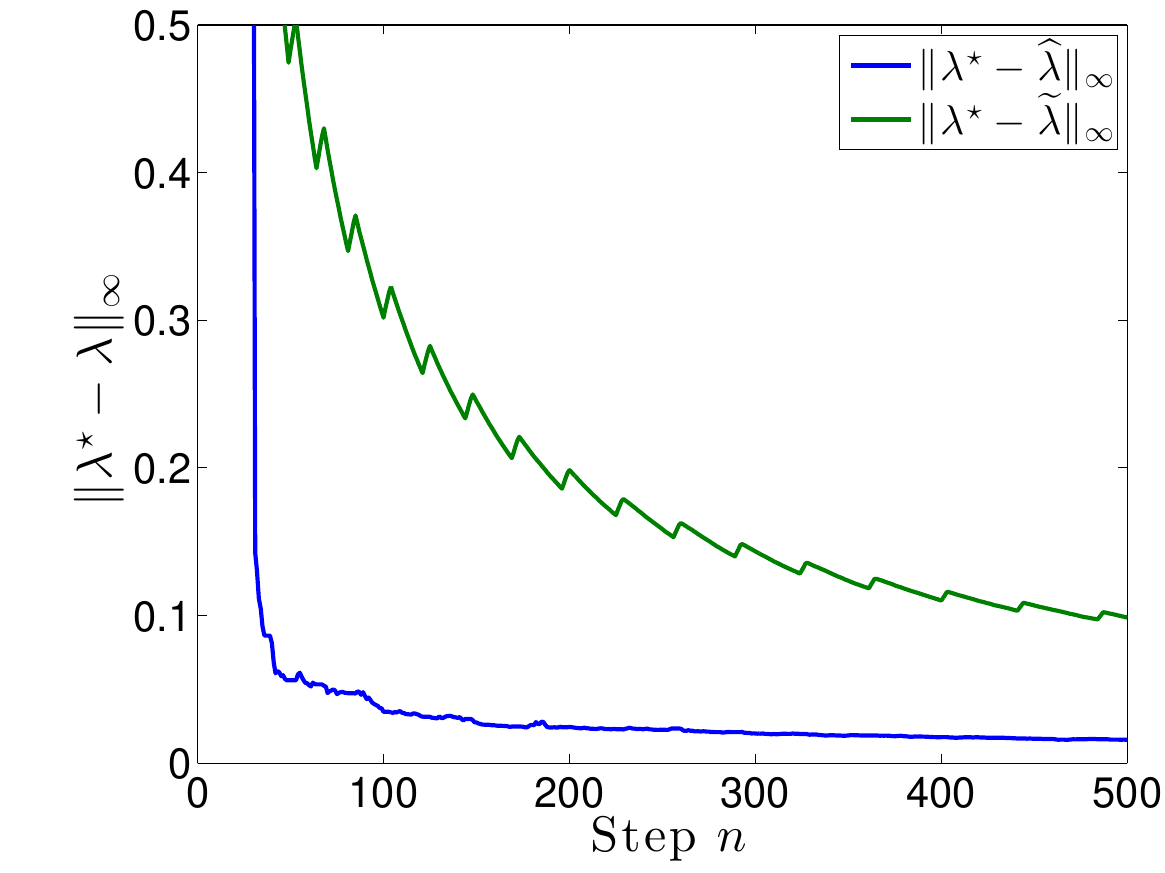}
\includegraphics[trim={1cm 0 0cm 0},width=0.32\columnwidth]{pics/ForcingAlgo_n10000_k1_track_arm4.pdf}
\includegraphics[trim={2cm 6.5cm 1cm 7.5cm},width=0.32\columnwidth]{pics/ForcingAlgo_n10000_k1_regret.pdf}
\includegraphics[trim={1cm 0 0cm 0},width=0.32\columnwidth]{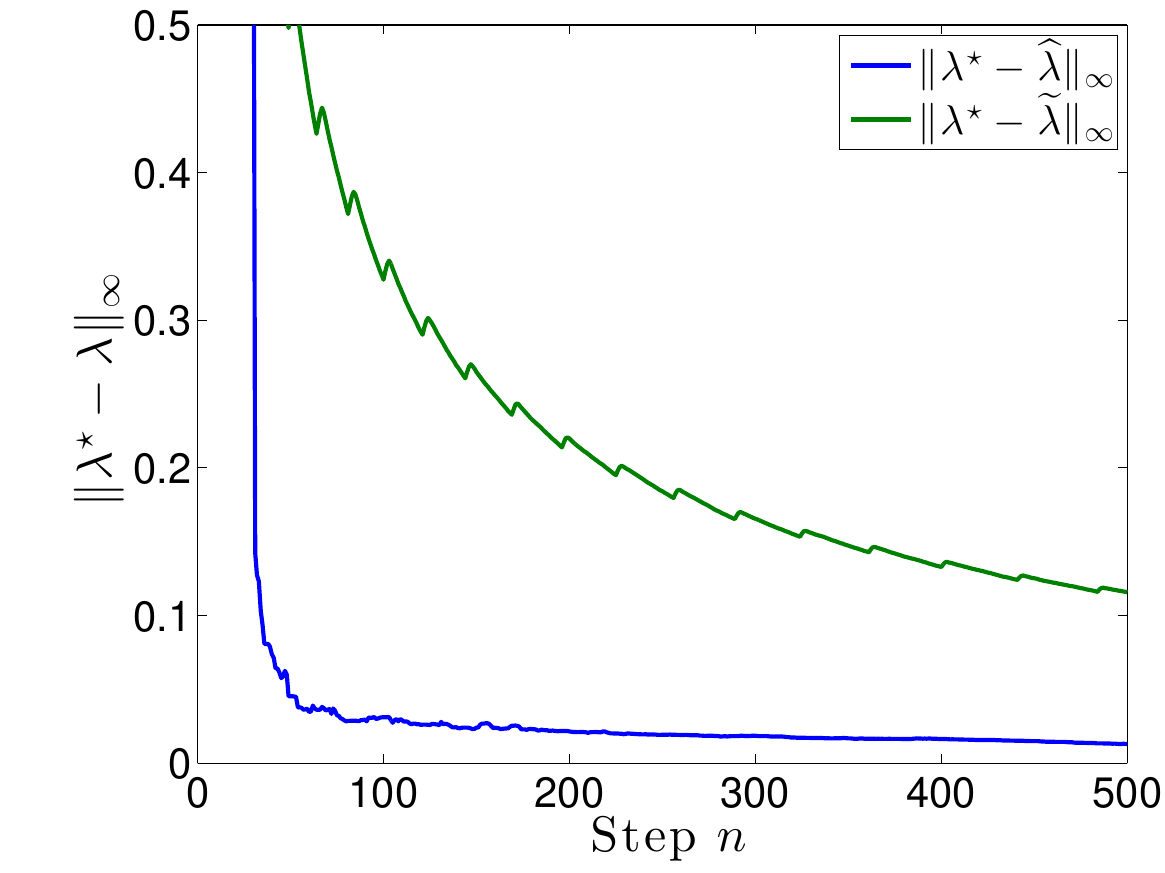}
\includegraphics[trim={1cm 0 0cm 0},width=0.32\columnwidth]{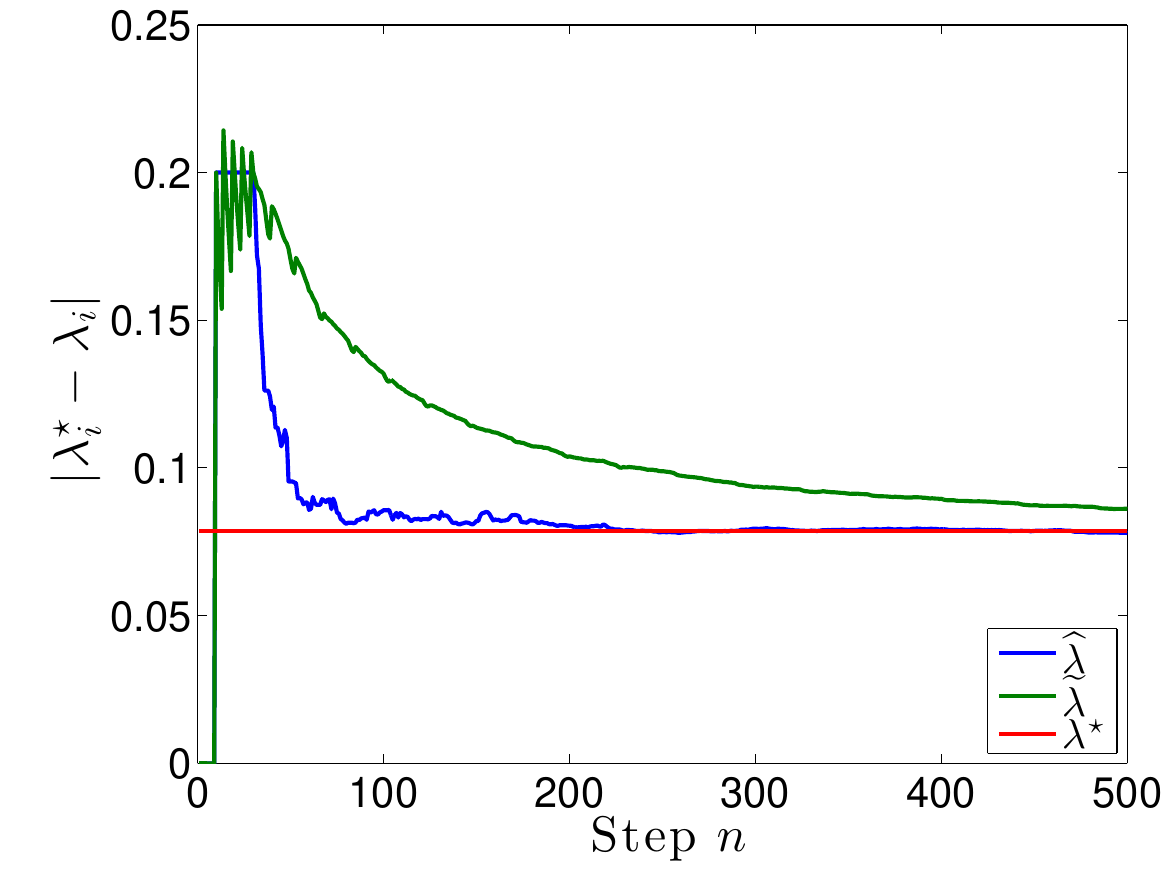}
\includegraphics[trim={2cm 6.5cm 1cm 7.5cm},width=0.32\columnwidth]{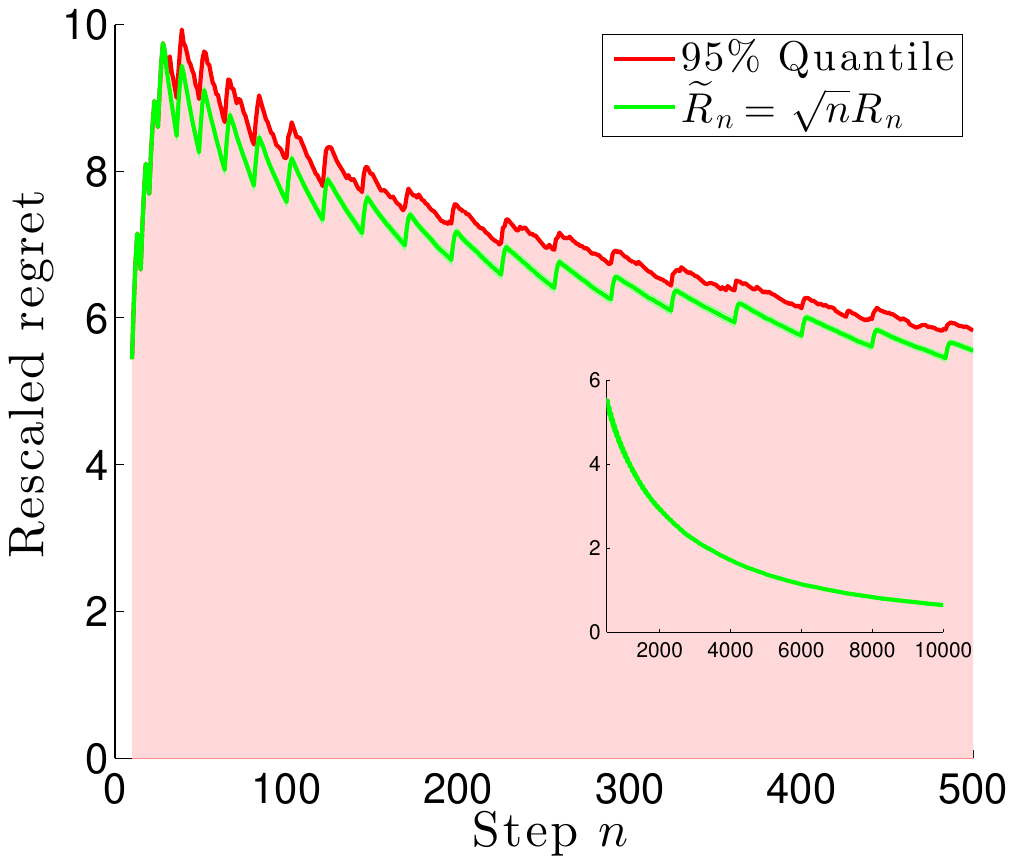}
\caption{Performance of \forcing with \textit{(top)} and without \textit{(bottom)} tracking step for the setting in Table~\ref{tab:exp.setting}. From left to right: $\ell_\infty$ error in approximating $\blambda^*$, error in approximating $\lambda_4^*$ and rescaled regret.}
\label{fig:tracking}
\end{figure}

\textbf{Tracking performance.}
Finally, we investigate the effect of the tracking strategy of \forcing by comparing it with an arm selection where $I_t$ is drawn at random from $\wh{\blambda}_t$. This version of \forcing does not try to compensate for the difference between the current allocation $\wt{\blambda}_t$ and the desired allocation $\wh{\blambda}_t$. We report the error in estimating $\blambda^*$ in $\ell_\infty$-norm for both $\wh{\blambda}_t$ and $\wt{\blambda}_t$ in the two configurations of \forcing. We can see in Fig.~\ref{fig:tracking}-\textit{(left/center)} that while $\wh{\blambda}$ is not affected by the tracking rule, $\wt{\blambda}$ is significantly slower in converging to $\blambda^*$ when the algorithm does not compensate for the mismatch is estimated optimal and actual allocations. Furthermore, this directly translates in a much higher regret as illustrated in Fig.~\ref{fig:tracking}-\textit{(right)}.

\end{document}